%% file: iclr2026_conference.tex
\documentclass{article} % For LaTeX2e
\usepackage{times}
\usepackage[preprint]{neurips_2026}

\input{math_commands.tex}

\usepackage[utf8]{inputenc} % allow utf-8 input
\usepackage[T1]{fontenc}    % use 8-bit T1 fonts
\usepackage{amsfonts}       % blackboard math symbols
\usepackage{nicefrac}       % compact symbols for 1/2, etc.
\usepackage{microtype}      % microtypography
\usepackage{xcolor}         % colors

\usepackage{hyperref}
\usepackage{url}
\usepackage{bbm}
\usepackage{amsthm}
\newtheorem{theorem}{Theorem}

\usepackage{booktabs}
\usepackage{multirow}
\usepackage{graphics}
\usepackage{graphicx}
\usepackage{booktabs}
\usepackage{multirow}
\usepackage{caption}
\usepackage{etoolbox}
\usepackage{makecell}
\usepackage{listings}
\usepackage{textcase}

\usepackage{graphicx}
\usepackage{pifont}% http://ctan.org/pkg/pifont
\usepackage{bbm}
\usepackage{bm}
\usepackage{caption}
\usepackage{amsmath}
\usepackage{amssymb}
\usepackage{mathtools}

\newtheorem{lemma}[theorem]{Lemma}
\theoremstyle{definition}

\usepackage{multirow}
\usepackage{adjustbox}
\usepackage[linesnumbered,ruled]{algorithm2e}
\usepackage{amssymb}
\usepackage{dirtytalk}
\usepackage{xurl}
\usepackage{thmtools}
\usepackage{thm-restate}
\usepackage{arydshln}
\usepackage{tablefootnote}
\usepackage{dirtytalk}
\usepackage{listings}
\usepackage{xcolor}
\usepackage{mathtools}

\title{Making Reconstruction FID Predictive of\\ Diffusion Generation FID}

\author{
Tongda Xu$^{1}$, Mingwei He$^{1}$, Shady Abu-Hussein$^{2}$, José Miguel Hernández-Lobato$^{2*}$, \\
$^{1}$Tsinghua University, $^2$University of Cambridge \AND Chunhang Zheng$^{3}$, Kai Zhao$^{4}$, Chao Zhou$^{4}$, Ya-Qin Zhang$^{1}$, Yan Wang$^{1}$\thanks{To whom correspondence should be addressed.} \\
$^{3}$Peking University, $^{4}$Kuaishou Technology
}

\begin{document}

\maketitle

\begin{abstract}
It is well known that the reconstruction FID (rFID) of a VAE is poorly correlated with the generation FID (gFID) of a latent diffusion model. We propose \textbf{interpolated FID (iFID)}, a simple variant of rFID that exhibits a strong correlation with gFID. Specifically, for each dataset element, we retrieve its nearest neighbor in latent space, interpolate between their latent representations, decode the interpolated latent, and compute the FID between the decoded samples and the original dataset. We provide an intuitive explanation for why iFID correlates well with gFID, and why reconstruction metrics can be negatively correlated with gFID, by connecting iFID to recent results on diffusion generalization and hallucination. Theoretically, we show that iFID evaluates decoded interpolations aligned with the ridge set around which diffusion samples concentrate, thereby measuring a quantity closely related to diffusion sample quality. Empirically, iFID is the first metric shown to strongly correlate with diffusion gFID across diverse VAEs, achieving Pearson and Spearman correlations of approximately $0.85$. The project page is available at \url{https://tongdaxu.github.io/pages/ifid.html}.
\end{abstract}

\begin{figure}[thb]
\centering
    \includegraphics[width=\linewidth]{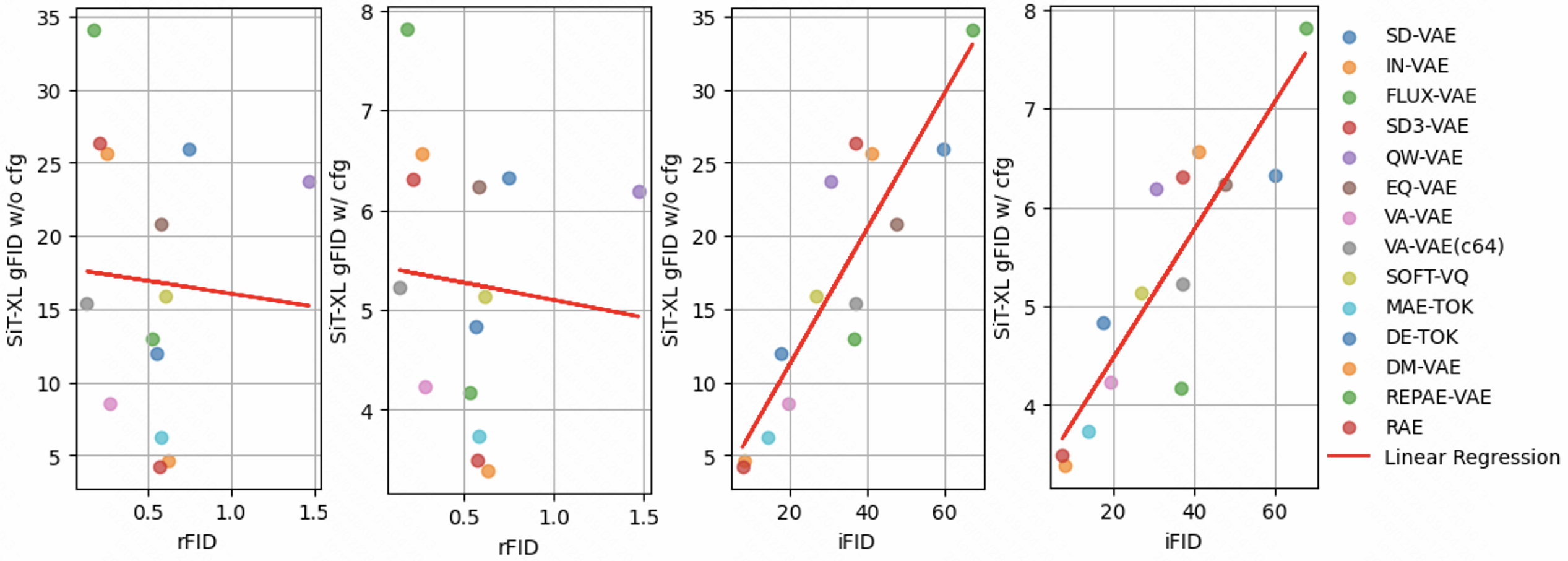}
\caption{\textit{Left two plots}: The rFID values of VAEs are uncorrelated with, or even negatively correlated with, the gFID values of diffusion models. \textit{Right two plots}: The iFID metric exhibits a strong positive correlation with the gFID values of diffusion models.}
\label{fig:mcov}
\end{figure}

\section{Introduction}

Variational autoencoders (VAEs) are foundational components of latent diffusion models (LDMs) \citep{rombach2022high,flux2024,esser2024scaling}. In an LDM, a VAE first maps images into a latent space, where a diffusion model is then trained to generate samples that are decoded back into pixel space. VAEs are typically optimized and evaluated by reconstruction quality \citep{rombach2022high,esser2024scaling}, \textit{e.g.} reconstruction Frechet Inception Distance (rFID) \citep{Heusel2017GANsTB}. Intuitively, a VAE with better reconstruction should preserve more image details and therefore lead to better generation. However, a phenomenon termed the \say{reconstruction--generation dilemma} has recently been widely observed \citep{Yao2025ReconstructionVG,Kouzelis2025EQVAEER,ye2025distribution,Skorokhodov2025ImprovingTD,Chen2025MaskedAA}: VAEs with excellent rFID can yield poor generation FID (gFID), whereas VAEs with worse rFID can achieve better generation performance.

In this work, we propose \textbf{interpolated FID (iFID)}, a simple variant of rFID that correlates strongly with gFID. Specifically, for each data point in the dataset, we identify its nearest neighbor in latent space. We then interpolate between the two latents and decode the interpolated latent into an image. Finally, we compute the FID between these decoded interpolated samples and the original dataset. By connecting iFID to recent results on diffusion generalization and hallucination, we provide an intuitive explanation for why iFID correlates well with diffusion sample quality, whereas reconstruction metrics can be negatively correlated with diffusion sample quality. Theoretically, we show that iFID measures diffusion sample quality by approximating samples from the ridge set. Empirically, we find that iFID is strongly correlated with diffusion gFID, achieving a correlation of approximately $0.85$.

In summary, our contributions are as follows:
\begin{itemize}
\item We propose interpolated FID (iFID), a simple variant of rFID based on nearest-neighbor latent interpolation. iFID is the first metric shown to strongly correlate with diffusion gFID, achieving high Pearson and Spearman correlations across a wide range of VAEs.
\item We provide an intuitive explanation for why iFID correlates well with generation performance whereas reconstruction metrics do not, by relating latent interpolation quality to recent results on diffusion generalization and hallucination.
\item Theoretically, we show that the ridge set manifold which diffusion samples concentrate around is close to the lines connecting nearby training samples. This explains why iFID is aligned with diffusion sample quality, as it evaluates decoded samples from the approximate ridge set.
\end{itemize}

\section{Preliminaries: Latent Diffusion Models}

The latent diffusion model \citep{rombach2022high} consists of a variational autoencoder (VAE) \citep{Kingma2013AutoEncodingVB}, which projects images into a latent space, and a diffusion model \citep{Ho2020DenoisingDP}, which generates samples in this latent space. Denote the source image as $X \sim p(X)$ and its latent representation as $Z \sim q(Z|X)$, where $q(Z|X)$ is the posterior distribution. The reconstructed image is given by $\hat{X} = g(Z)$, where $g(\cdot)$ denotes the decoder. A diffusion model is trained in the latent space to approximate the latent distribution $p(Z)$, and generated latents are subsequently decoded by $g(\cdot)$. The forward process of a diffusion model constructs a Markov chain by incrementally adding noise over timesteps $t \in [0,T]$. More specifically, given diffusion parameters $\alpha_t$ and $\sigma_t^2$, the conditional forward kernel is
\begin{gather}
    q(Z_t|Z) = \mathcal{N}(\alpha_t Z,\sigma_t^2I).
\end{gather}
The reverse process of a diffusion model is described by the reverse-time stochastic differential equation (SDE) \citep{Anderson1982ReversetimeDE}. Sampling from the diffusion model involves simulating the reverse SDE from timestep $T$ to $0$ \citep{Song2020ScoreBasedGM}, which depends on the score $\nabla_{Z_t}\log q(Z_t)$. Therefore, a diffusion model is commonly parameterized by a score estimator $s_{\theta}(Z_t, t)$ that approximates $\nabla_{Z_t}\log q(Z_t)$ and is optimized by minimizing the denoising score matching loss \citep{Vincent2011ACB}:
\begin{gather}
    \mathcal{L}_\text{DSM} = \mathbb{E}\|s_{\theta}(Z_t,t) - \nabla_{Z_t}\log q(Z_t|Z)\|^2. \label{eq:dsm}
\end{gather}

\section{Making Reconstruction FID Predictive of Generation FID}
\subsection{Reconstruction FID, Generation FID, and Interpolated FID}
We denote the images in the validation dataset by $x^{(1:N)}$, their latent representations by $z^{(1:N)}$, the Frechet Inception Distance (FID) between two sets of images by $d_{\textup{FID}}(\cdot,\cdot)$, the decoder by $g(\cdot)$, the diffusion solver by $\Phi(\cdot,\cdot)$, and Gaussian noise by $\epsilon \sim \mathcal{N}(0,I)$. Reconstruction FID (rFID) and generation FID (gFID) are defined as
\begin{gather}
    \textup{rFID}=d_{\textup{FID}}(x^{(1:N)},g(z^{(1:N)})),
    \textup{gFID}= d_{\textup{FID}}(x^{(1:N)},g(\Phi(\epsilon^{(1:M)},T))).
\end{gather}
It is well known that rFID does not correlate well with gFID \citep{Kouzelis2025EQVAEER, Yao2025ReconstructionVG}. This phenomenon is often referred to as the \say{reconstruction--generation trade-off}. In this paper, we propose a remarkably simple variant of rFID that exhibits a strong correlation with gFID. Specifically, instead of directly evaluating the original latents $z^{(1:N)}$, we linearly interpolate between each latent $z^{(i)}$ and its nearest neighbor $z^{(i^*)}$ in latent space:
\begin{gather}
    \textup{iFID}:=d_{\textup{FID}}(x^{(1:N)},g(\hat{z}^{(1:N)}))\textup{, where } \hat{z}^{(i)} = \frac{z^{(i)}+z^{(i^*)}}{2},
    z^{(i^*)} := \arg\min_{j\neq i}||z^{(j)}-z^{(i)}||.
\end{gather}
We refer to this metric as \textbf{interpolated FID (iFID)}, since it evaluates the sample quality of decoded interpolated latents. As shown in Figure~\ref{fig:mcov}, iFID correlates strongly with gFID.
\begin{figure}[thb]
\centering
    \includegraphics[width=\linewidth]{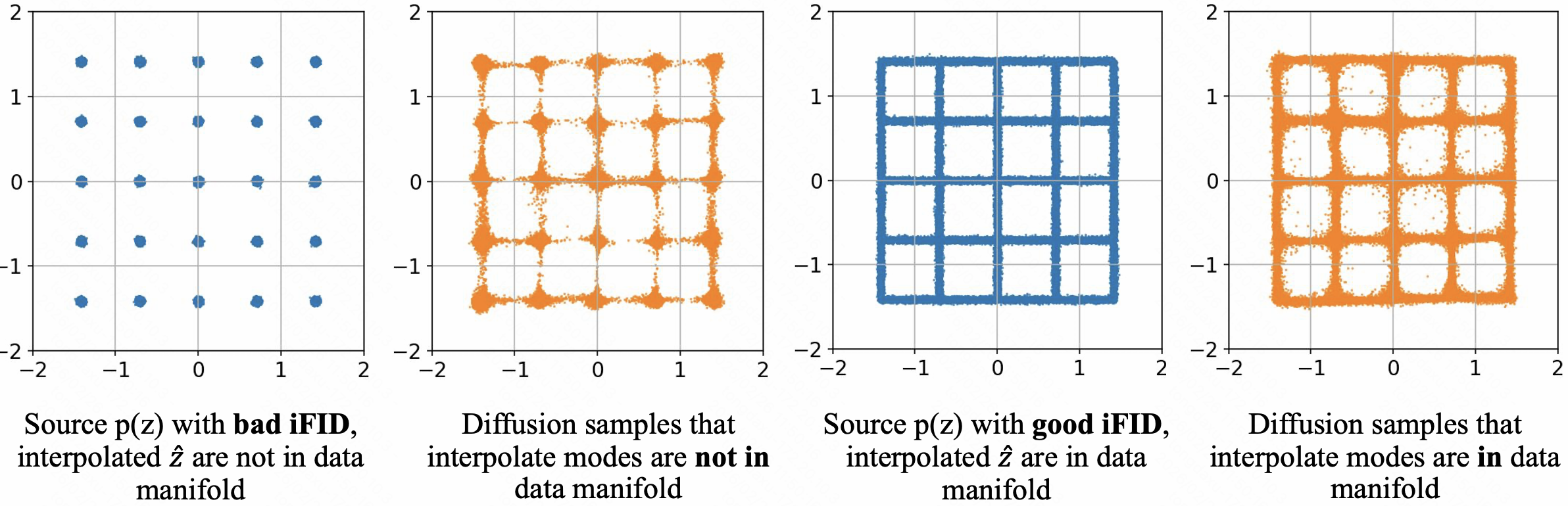}
\caption{Toy example illustrating how latent-space geometry affects diffusion samples.
\textit{Left two plots}: An isolated 25-component Gaussian mixture arranged on a two-dimensional square grid has poor iFID because the interpolated latent $\hat{z}$ lies off the data manifold. Diffusion samples interpolating between nearby modes also fall outside the data manifold, leading to hallucination. \textit{Right two plots}: A connected 25-component Gaussian mixture has good iFID because the interpolated latent $\hat{z}$ remains on the data manifold. Diffusion samples interpolating between nearby modes therefore stay within the data manifold, reducing hallucination.}
\label{fig:ty}
\end{figure}
\subsection{Why Interpolated FID Predicts Sample Quality: Intuition and Examples}
\textbf{Diffusion models generate unseen samples by interpolating training data.}
To understand why iFID is related to diffusion sample quality, we first revisit the training objective of diffusion models. Specifically, for a fixed training dataset $z^{(1:M)}$, the score matching objective admits a closed-form optimum known as the empirical score.
\begin{lemma}[Denoising score matching has an optimal solution {\citep{Song2025SelectiveUI}}]
Given the forward diffusion process $z_t=\alpha_t z + \sigma_t\epsilon$ and a training set $\{z^{(1)},\ldots,z^{(M)}\}$, denoising score matching in Eq.~\ref{eq:dsm} has the following empirical optimal solution:
\begin{gather}
    s_*(z_t,t) = \frac{1}{\sigma_t^2}(-z_t+\alpha_t\sum_{i=1}^Mw^{(i)}z^{(i)}),\textup{ where } w^{(i)} = \frac{\exp{(-||z_t-\alpha_tz^{(i)}||^2/2\sigma_t^2)}}{\sum_{j=1}^M\exp{(-||z_t-\alpha_tz^{(j)}||^2/2\sigma_t^2)}}. \label{eq:odsm}
\end{gather}
\end{lemma}
It is well established that when the score estimator $s_\theta(z_t, t)$ matches the empirical optimal score $s_*(z_t, t)$, a diffusion model tends to replicate samples from the training dataset \citep{Kamb2024AnAT,Bonnaire2025WhyTraining,Buchanan2025OnTE}. Existing literature suggests that the generalization ability of diffusion models, namely their capacity to generate unseen samples, arises from a mismatch between the learned score and the empirical score \citep{Kadkhodaie2023GeneralizationID,Somepalli2022DiffusionAO,Yoon2023DiffusionMemorize,Buchanan2025OnTE,Song2025SelectiveUI}.

How are newly generated samples related to the training dataset? Prior studies suggest that diffusion-generated samples can often be interpreted as interpolations or compositions of training images \citep{Okawa2023CompositionalAE}. In particular, \citep{Kamb2024AnAT,Somepalli2022DiffusionAO} show that diffusion samples can behave as local combinations of training examples, while \citep{Aithal2024UnderstandingHI,Deschenaux2024GoingBC,ChandranC2025LaplacianSS} demonstrate that diffusion models can interpolate between modes of the training distribution.

\textbf{iFID measures the validity of interpolated data.}
Under the hypothesis that diffusion models generate new samples by interpolating and composing images from the training dataset, iFID should intuitively correlate with diffusion sample quality because it evaluates the validity of interpolated latent representations. A low iFID indicates that decoded interpolated latents are close to the source data distribution in FID, suggesting that diffusion samples produced through similar interpolations are more likely to remain on the data manifold. Thus, a latent space with low iFID can help reduce hallucinations caused by mode interpolation \citep{Aithal2024UnderstandingHI,ChandranC2025LaplacianSS}.

\textbf{Toy example with a 2D Gaussian mixture.}
To better illustrate the relationship between iFID and sample quality, we provide a toy example with two different latent-space geometries in Figure~\ref{fig:ty}. We first consider a 25-component Gaussian mixture arranged on a two-dimensional square grid, where the components are well separated. This isolated latent distribution has poor iFID because the interpolated samples $\hat{z}$ lie off the data manifold. Consequently, diffusion samples that interpolate between nearby modes also fall outside the data manifold, resulting in hallucinations. In contrast, a connected 25-component Gaussian mixture has good iFID because the interpolated samples $\hat{z}$ remain on the data manifold. Since the latent distribution is connected, diffusion samples are less likely to leave the data manifold, thereby reducing hallucinations.

\subsection{Why Interpolated FID Predicts Sample Quality: Ridge-Set Theory}
Beyond the preceding intuition and examples, we provide a theoretical perspective on why iFID predicts sample quality by connecting it to ridge-set theory \citep{He2026DiffusionMG}. Prior work \citep{He2026DiffusionMG} shows that diffusion samples concentrate around a manifold known as the ridge set, which is determined by the Jacobian of the optimal score function. We show that, under a local two-nearest-neighbor approximation in the small-noise regime, the ridge set near a location $z_t$ is approximately aligned with the line connecting the two nearest neighbors of $z_t$ in the training set. Since iFID evaluates decoded interpolations along these local directions, it provides a theoretical proxy for diffusion sample quality.

\begin{figure}[thb]
\centering
    \includegraphics[width=0.5\linewidth]{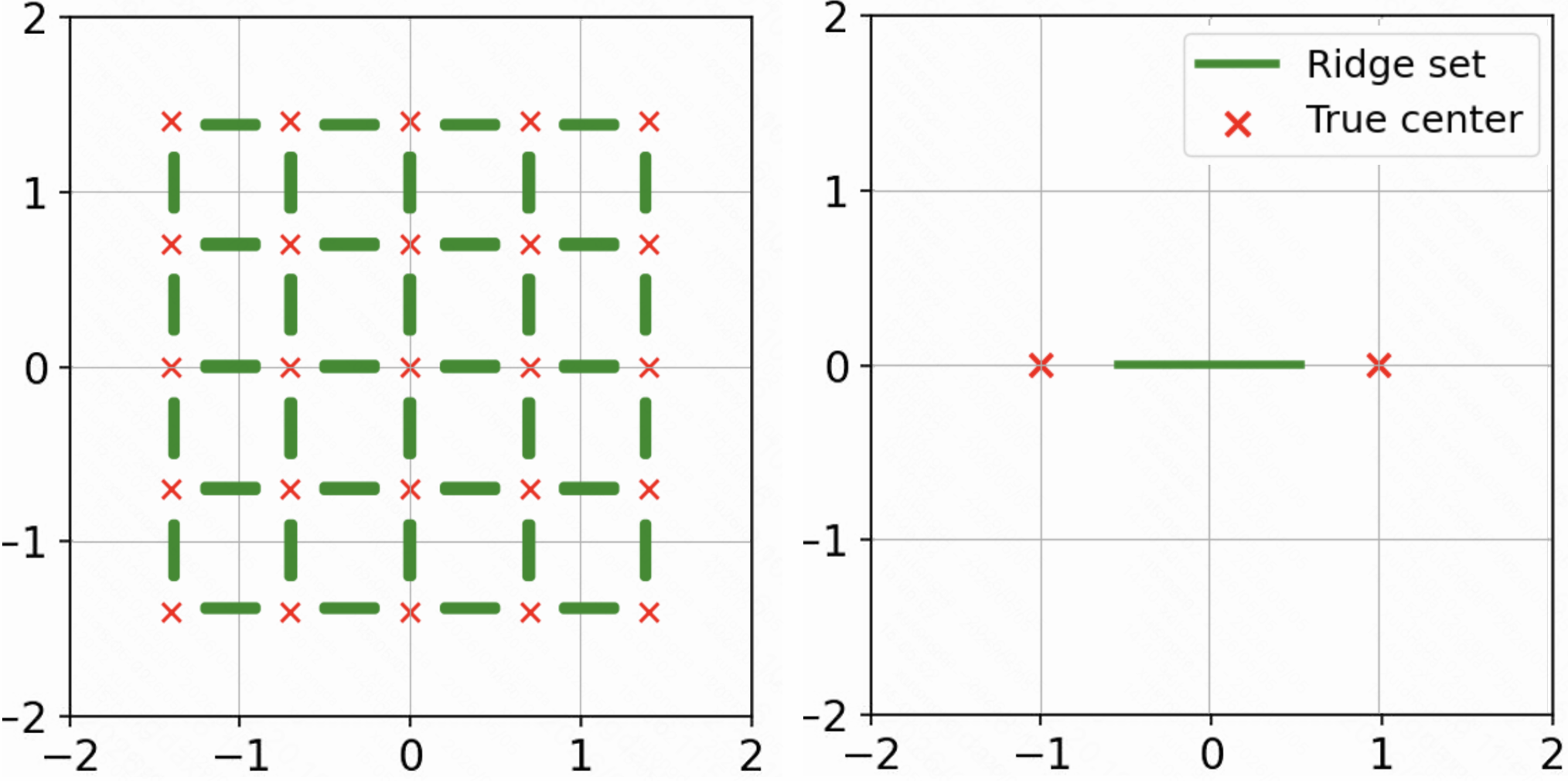}
\caption{Ridge sets of the source distributions $p(z)$ from Figure~\ref{fig:ty} and Figure~\ref{fig:ty2}, computed with $\beta=0.0$ and $t=0.1$. The ridge sets approximate the manifolds around which diffusion samples concentrate and visually resemble lines connecting nearby training samples.}
\label{fig:ty3}
\end{figure}

The ridge set is motivated by the observation that the learned score estimator $s_{\theta}(z_t,t)$ may capture only a low-rank approximation to the Jacobian of the empirical score $s_*(z_t,t)$. This Jacobian discrepancy prevents diffusion samples from simply replicating the training set and instead causes them to concentrate around a manifold known as the ridge set.
\begin{lemma}[Diffusion samples concentrate around the ridge set determined by the Jacobian {\citep{He2026DiffusionMG}}]
Diffusion samples $z_t$ concentrate around the ridge set
\begin{gather}
    \mathcal{R}(\beta,t) = \{z_t|V(z_t,t,\beta)V(z_t,t,\beta)^Ts_*(z_t) = 0\},
\end{gather}
where $\beta$ is a hyperparameter and $V(z_t,t,\beta)$ is the matrix whose columns are the $d-d^*$ eigenvectors of the Jacobian $J$ corresponding to the smallest $d-d^*$ eigenvalues $-\beta \ge \lambda_{d^*+1} \ge \cdots \ge \lambda_d$. The Jacobian is defined as
\begin{gather}
    J := \nabla s_*(z_t,t) = -\frac{1}{\sigma_t^2}I + \frac{\alpha_t^2}{\sigma_t^4} \sum_{i=1}^M w^{(i)}(z^{(i)}-\mu)(z^{(i)}-\mu)^T,
    \textup{where } \mu = \sum_{j=1}^M w^{(j)}z^{(j)}.
\end{gather}
\end{lemma}

Figure~\ref{fig:ty3} shows examples of ridge sets for the 25-component and 2-component Gaussian mixtures. We observe that these ridge sets often resemble lines connecting nearby training samples. Although this observation is consistent with the intuition that diffusion models interpolate between nearby modes \citep{Aithal2024UnderstandingHI,ChandranC2025LaplacianSS}, its theoretical grounding remains limited. We provide such an explanation by showing that, in the small-$t$ regime, the ridge set near a location $z_t$ is close to the line connecting the two nearest neighbors of $z_t$ in the training set.

\begin{theorem}[The Jacobian is close to the two-point Jacobian when $t$ is small]
When $t$ is small, the full Jacobian is close to the Jacobian determined by
$z^{(i^*)}$ and $z^{(j^*)}$, the two nearest-neighbor training samples to $z_t$.
Assume that $B=\max_k\|z^{(k)}\| < \infty$ and $\Delta = \min_{k\neq i^*,j^*} ||z_t-\alpha_tz^{(k)}||^2 - ||z_t-\alpha_tz^{(j^*)}||^2 > 0$. Then,
\begin{align}
    \|J-J_2\| &\le \frac{\alpha_t^2}{\sigma_t^4}
    \left(9B^2(M-2)\right)\exp{\left(-\frac{\Delta}{2\sigma_t^2}\right)},\notag \\ 
    \textup{where } J_2 &:= -\frac{1}{\sigma_t^2}I
    + \frac{\alpha_t^2}{\sigma_t^4}\sum_{k=i^*,j^*}\tilde{w}^{(k)}
    (z^{(k)}-\tilde{\mu})(z^{(k)}-\tilde{\mu})^T.
\end{align}
Here, $J_2$ is the Jacobian of the optimal score for the two-point training set $\{z^{(i^*)},z^{(j^*)}\}$, where 
\begin{gather}
    \tilde{w}^{(k)} = \frac{w^{(k)}}{\sum_{l=i^*,j^*}w^{(l)}},\tilde{\mu}=\sum_{k=i^*,j^*}\tilde{w}^{(k)}z^{(k)}.
\end{gather}
\end{theorem}
\begin{theorem}[The ridge set of the two-point Jacobian is aligned with the line connecting the two points]
\label{thm:rl}
In the two-point approximation, the ridge set induced by $J_2$ is the line connecting
$z^{(i^*)},z^{(j^*)}$:
\begin{gather}
    \mathcal{R}\left(\beta=\frac{1}{\sigma_t^2},t\right)
    =
    \left\{z_t \mid z_t = (1-\gamma)z^{(i^*)} + \gamma z^{(j^*)},\ \gamma \in \mathbb{R}\right\}.
\end{gather}
Moreover, the nontrivial eigenvalue of $J_2$ is maximized at $\gamma=0.5$ and $z_t=(z^{(i^*)}+z^{(j^*)})/2$.
\end{theorem}
Theorem~\ref{thm:rl} provides a theoretical interpretation of why iFID correlates well with diffusion sample quality. In the local two-point approximation, the interpolated latent $\hat{z}^{(i)}$ lies on the approximate ridge set. Thus, iFID evaluates decoded samples from a manifold around which diffusion samples tend to concentrate. Moreover, at the midpoint $(z^{(i^*)}+z^{(j^*)})/2$, the only nontrivial eigenvalue of $J_2$ attains its maximum value. This is consistent with our observation that iFID achieves its strongest correlation with gFID when using midpoint interpolation between nearby training samples.
\begin{figure}[thb]
\centering
    \includegraphics[width=\linewidth]{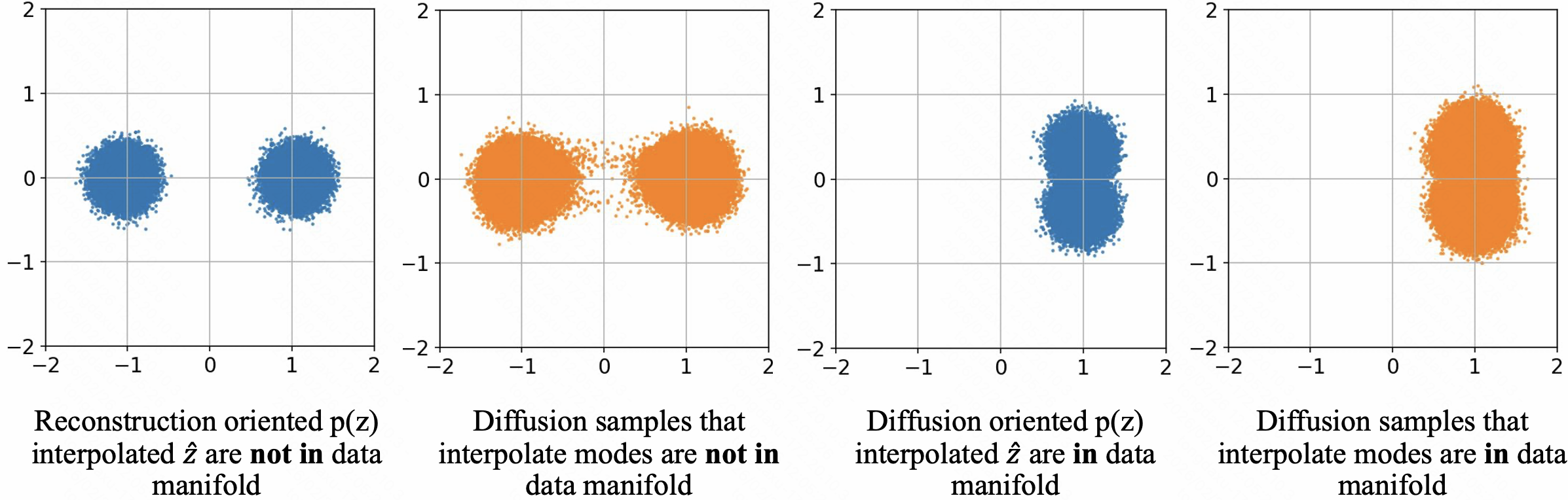}
\caption{Toy example illustrating the difference between reconstruction-oriented and diffusion-oriented latent spaces. \textit{Left two plots}: The latent distribution is an isolated two-mode Gaussian mixture with poor interpolatability, since the interpolated latent $\hat{z}$ does not lie on the data manifold. In this case, diffusion samples interpolating between nearby modes also fall outside the data manifold, leading to hallucinations. \textit{Right two plots}: The latent distribution is an overlapping two-mode Gaussian mixture with good interpolatability, as the interpolated latent $\hat{z}$ remains on the data manifold. Consequently, diffusion samples interpolating between nearby modes stay within the data manifold, reducing hallucinations.}
\label{fig:ty2}
\end{figure}
\subsection{Why Reconstruction Can Correlate Negatively with Sample Quality}
The preceding analysis suggests that diffusion models benefit from interpolatable and connected latent spaces, since such spaces allow interpolated samples to remain close to the data manifold. This perspective also helps explain why reconstruction metrics can be negatively correlated with diffusion sample quality. Reconstruction-based objectives tend to favor latent spaces in which different inputs are well separated, making it easier for the decoder to recover each input accurately. However, such separated or disconnected latent spaces can be unfavorable for diffusion generation, since samples interpolating between nearby modes may leave the data manifold. This tension gives rise to the so-called \say{reconstruction--generation dilemma} \citep{Yao2025ReconstructionVG}.

To clarify this intuition, we provide a toy example with two Gaussian mixture models (GMMs) in Figure~\ref{fig:ty2}. Specifically, we consider a fully factorized Gaussian VAE trained on two data points $x^{(1:2)}$, with fixed posterior variance $\sigma^2$ and KL divergence. We study the problem of optimizing the posterior means $\mu^{(i)}$ to minimize the reconstruction loss. Minimizing reconstruction error encourages the posterior means $\mu^{(i)}$ to be as separable as possible, making it easier for the decoder to associate each input $x^{(i)}$ with its corresponding posterior sample $z^{(i)}$. However, such a separable latent space can be unfavorable for diffusion generation. As shown in Figure~\ref{fig:ty2}, diffusion samples may interpolate between the two separated modes, and these intermediate samples can fall outside the data manifold, resulting in hallucinations. In contrast, when the latent modes substantially overlap, interpolated diffusion samples are more likely to remain on the data manifold, reducing hallucinations. This comes at the cost of reconstruction quality, since overlapping latent modes are harder for the decoder to distinguish.
\begin{table}[thb]
\caption{VAEs included in our study. We evaluate 13 VAEs with publicly available checkpoints and train corresponding diffusion models in their respective latent spaces.}
\label{tab:vaes}
\centering
\resizebox{0.85\linewidth}{!}{
\begin{tabular}{@{}lccc@{}}
\toprule
VAE Name & Latent Dim. & Arch. & Training  \\ \midrule
SD-VAE \citep{rombach2022high} & $4\times 32 \times 32$ & UNet & Recon.  \\
IN-VAE \citep{Yao2025ReconstructionVG} & $32\times 16 \times 16$ & UNet & Recon.  \\
FLUX-VAE \citep{flux2024} & $16\times 32\times 32$ & UNet & Recon. \\
QwenImage-VAE \citep{wu2025qwen} & $16\times 32 \times 32$ & UNet & Recon. \\
SD3-VAE \citep{esser2024scaling} & $16\times 32 \times 32$ & UNet & Recon.  \\
EQ-VAE \citep{Kouzelis2025EQVAEER} & $4\times 32 \times 32$ & UNet & Recon. + Equivariance  \\
VA-VAE \citep{Yao2025ReconstructionVG} & $32\times 16 \times 16$ & UNet & Recon. + DINO alignment  \\
SOFT-VQ \citep{chen2025softvq} & $64\times 32$ & ViT & Recon. + DINO alignment  \\
MAE-TOK \citep{Chen2025MaskedAA} & $128\times 32$ & ViT & Recon. + Mask + DINO alignment  \\
DE-TOK \citep{yang2025latent} & $128\times 32$ & ViT & Recon. + Mask + latent denoising  \\
DM-VAE \citep{ye2025distribution} & $256\times 32$ & ViT & Recon. + Distribution matching  \\
REPAE-VAE \citep{Leng2025REPAEUV} & $4\times 32 \times 32$ & UNet & Recon. + REPA loss  \\
RAE \citep{zheng2025diffusion} & $768\times 16 \times 16$ & ViT & DINO encoder + recon. decoder  \\
\bottomrule
\end{tabular}
}
\end{table}

\section{Experimental Results}
\subsection{Experimental Setup}
\textbf{Dataset and metrics.}
We conduct experiments on the $256\times 256$ ImageNet dataset. All diffusion models are trained on the ImageNet training split, and all metrics are evaluated on the ImageNet validation split. As reconstruction metrics, we consider PSNR, LPIPS, SSIM, and rFID, which are standard metrics for evaluating VAEs. As non-reconstruction metrics, we consider diffusion loss, as well as objective functions designed to improve the diffusion performance of VAEs, including EQ Loss, SE Loss, VF Loss, and GMM Loss \citep{Kouzelis2025EQVAEER, Skorokhodov2025ImprovingTD, Yao2025ReconstructionVG, chen2025aligning}. To quantify how well each metric predicts diffusion sample quality, we report the Pearson correlation coefficient (PCC) and Spearman rank correlation coefficient (SRCC) between each metric and gFID.

\textbf{VAEs and diffusion models.}
We evaluate 13 representative VAEs with publicly available checkpoints; details are provided in Table~\ref{tab:vaes}. These VAEs cover diverse latent dimensions, architectures, and training objectives. Some, such as SD-VAE, IN-VAE, QW-VAE, FLUX-VAE, and SD3-VAE \citep{rombach2022high,Yao2025ReconstructionVG,flux2024,esser2024scaling,wu2025qwen}, are optimized solely for reconstruction. EQ-VAE additionally incorporates equivariance regularization \citep{Kouzelis2025EQVAEER}. Other VAEs, including VA-VAE, SOFT-VQ, MAE-TOK, DE-TOK, DM-VAE, REPAE-VAE, and RAE, employ contrastive-learning-based image encoders or representation-alignment objectives \citep{Yao2025ReconstructionVG,chen2025softvq,chen2025aligning,yang2025latent,ye2025distribution,Leng2025REPAEUV,zheng2025diffusion}.

For each VAE, we train two diffusion models of different sizes, SiT-B and SiT-XL, in the corresponding latent space. All diffusion models are trained for 40 epochs using the Adam optimizer. We report gFID both with and without classifier-free guidance (CFG).

\begin{table}[thb]
\caption{Correlation between VAE metrics and diffusion gFID. Correlations are computed using metrics obtained from SiT models trained with 13 pre-trained VAEs. EQ Loss and SE Loss are not evaluated for VAEs with 1D latents. Following \citep{zheng2025diffusion}, RAE is not evaluated with SiT-B. PCC denotes the Pearson correlation coefficient, and SRCC denotes Spearman's rank correlation coefficient. \textbf{Bold}: best. $^*$: $p<0.05$.}
\label{tab:resultfid}
\centering
\resizebox{0.9\linewidth}{!}{
\begin{tabular}{@{}lcccccccc@{}}
\toprule
\multirow{3}{*}{Metric} & \multicolumn{4}{c}{gFID SiT-B} & \multicolumn{4}{c}{gFID SiT-XL} \\ \cmidrule(lr){2-5}\cmidrule(lr){6-9} 
 & \multicolumn{2}{c}{w/o CFG} & \multicolumn{2}{c}{w/ CFG} & \multicolumn{2}{c}{w/o CFG} & \multicolumn{2}{c}{w/ CFG} \\ \cmidrule(lr){2-3}\cmidrule(lr){4-5}\cmidrule(lr){6-7}\cmidrule(lr){8-9} 
 & PCC$\uparrow$ & SRCC$\uparrow$ & PCC$\uparrow$ & SRCC$\uparrow$ & PCC$\uparrow$ & SRCC$\uparrow$ & PCC$\uparrow$ & SRCC$\uparrow$ \\ \midrule
\multicolumn{9}{@{}l@{}}{\textit{Reconstruction metrics}} \\
-PSNR & -0.81$^*$ & -0.81$^*$ & -0.83$^*$ & -0.82$^*$ & -0.79$^*$ & -0.78$^*$ & -0.79$^*$ & -0.85$^*$ \\
-SSIM & -0.78$^*$ & -0.81$^*$ & -0.80$^*$ & -0.84$^*$ & -0.77$^*$ & -0.78$^*$ & -0.77$^*$ & -0.85$^*$ \\
LPIPS & -0.73$^*$ & -0.74$^*$ & -0.72$^*$ & -0.76$^*$ & -0.73$^*$ & -0.72$^*$ & -0.73$^*$ & -0.79$^*$ \\
rFID & -0.04 & -0.31 & -0.07 & -0.31 & -0.06 & -0.21 & -0.15 & -0.31 \\
\midrule
\multicolumn{9}{@{}l@{}}{\textit{Non-reconstruction metrics}} \\
Diffusion Loss & 0.21 & 0.24 & 0.22 & 0.20 & 0.34 & 0.37 & 0.24 & 0.29 \\ 
EQ Loss \citep{Kouzelis2025EQVAEER} & -0.62 & -0.52 & -0.78$^*$ & -0.76$^*$ & -0.39 & -0.15 & -0.37 & -0.13 \\
SE Loss \citep{Skorokhodov2025ImprovingTD} & -0.70$^*$ & -0.70$^*$ & -0.75$^*$ & -0.83$^*$ & -0.77$^*$ & -0.70 & -0.79$^*$ & -0.79 \\
VF Loss \citep{Yao2025ReconstructionVG} & 0.10 & 0.05 & 0.09 & 0.03 & 0.17 & 0.11 & 0.09 & -0.01 \\
GMM Loss \citep{Chen2025MaskedAA} & 0.25 & 0.04 & 0.25 & 0.05 & 0.28 & 0.01 & 0.23 & -0.02 \\
iFID (ours) & \textbf{0.85}$^*$ & \textbf{0.86}$^*$ & \textbf{0.82}$^*$ & \textbf{0.84}$^*$ & \textbf{0.89}$^*$ & \textbf{0.91}$^*$ & \textbf{0.88}$^*$ & \textbf{0.92}$^*$ \\ \bottomrule
\end{tabular}
}
\end{table}
\begin{figure}[thb]
\centering
  \includegraphics[width=0.9\linewidth]{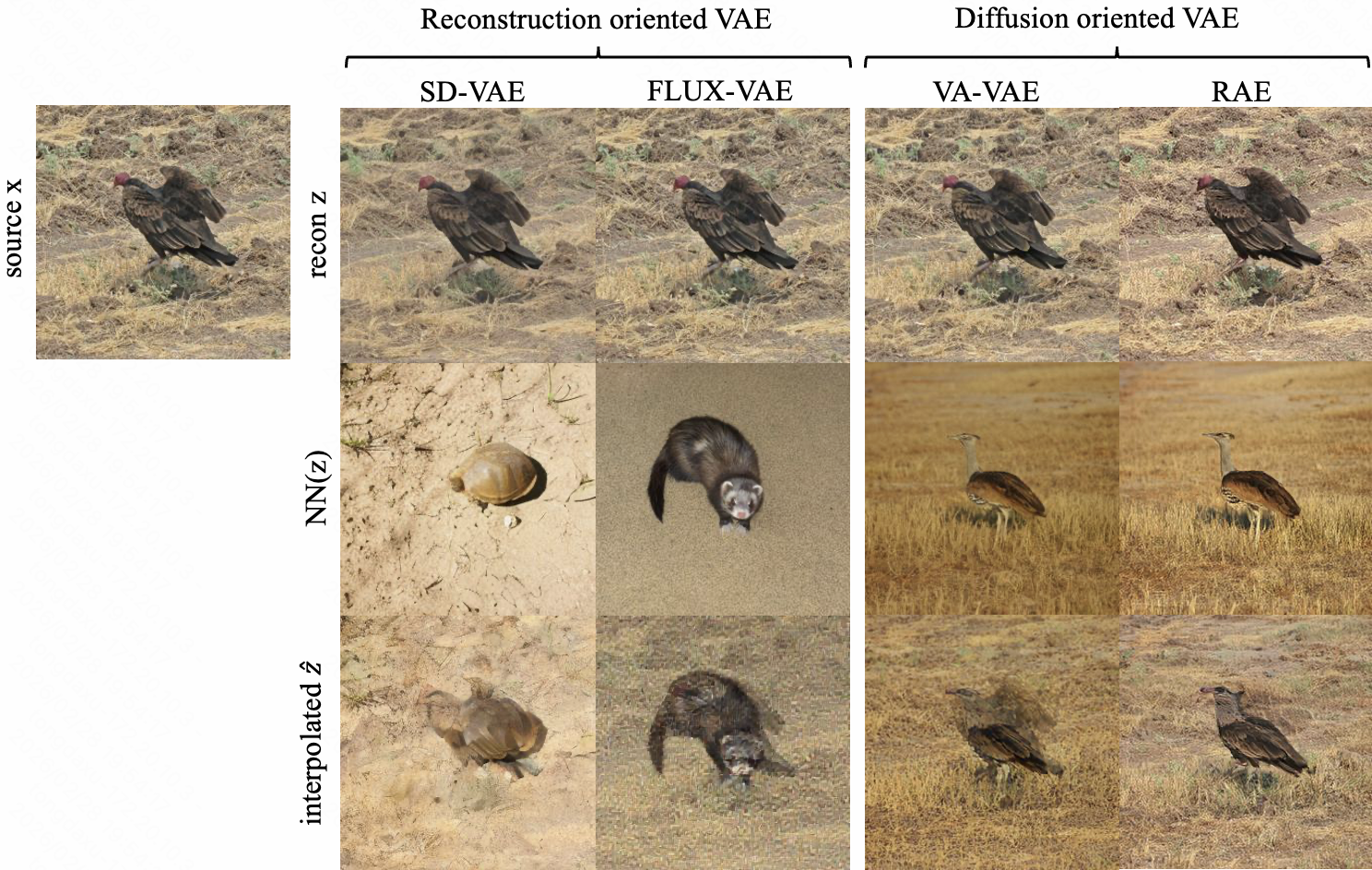}
\caption{Visualization of nearest-neighbor latents $\mathrm{NN}(z)$ and interpolated latents $\hat{z}$. For reconstruction-oriented VAEs, $\mathrm{NN}(z)$ is often semantically different from $z$, causing the decoded interpolated latents to produce invalid images. In contrast, for diffusion-oriented VAEs, $\mathrm{NN}(z)$ is semantically similar to $z$, and the decoded interpolated latents produce realistic images.}
\label{fig:viz}
\end{figure}

\subsection{Main Results}

\textbf{Reconstruction metrics correlate negatively with diffusion sample quality.}
As shown in Table~\ref{tab:resultfid}, image-to-image distortion metrics, including PSNR, SSIM, and LPIPS, are strongly negatively correlated with gFID. This observation supports the \say{reconstruction--generation dilemma}: better reconstruction quality does not necessarily translate into better diffusion generation quality.

\textbf{iFID exhibits a significantly stronger correlation with diffusion sample quality.}
Among non-reconstruction metrics, diffusion loss, VF Loss \citep{Yao2025ReconstructionVG}, and GMM Loss \citep{chen2025aligning} show only moderate or weak correlations with gFID. In contrast, iFID exhibits a substantially stronger correlation with gFID, achieving PCC and SRCC values close to $0.9$. Notably, iFID correlates much more strongly with gFID than diffusion loss, suggesting that diffusion loss alone is not a reliable predictor of sample quality when comparing different latent spaces.

\textbf{Visualization of nearest-neighbor and interpolated latents.}
Figure~\ref{fig:viz} visualizes the decoded latent $z$, its nearest-neighbor latent $\mathrm{NN}(z)$, and the interpolated latent $\hat{z}$ for different VAEs. For reconstruction-optimized VAEs, such as SD-VAE and FLUX-VAE \citep{rombach2022high,flux2024}, $\mathrm{NN}(z)$ is often semantically unrelated to $z$, causing the decoded interpolated latent $\hat{z}$ to produce invalid images. In contrast, for diffusion-optimized VAEs, such as VA-VAE and RAE \citep{Yao2025ReconstructionVG,zheng2025diffusion}, $\mathrm{NN}(z)$ is semantically similar to $z$, and the decoded interpolated latent $\hat{z}$ produces realistic images. Similar observations have also been reported in \citep{peng2022beit,Song2025SelectiveUI}.

\begin{table}[thb]
\caption{Sensitivity analysis of the latent interpolation method, the number of images used for nearest-neighbor search, and the number of nearest neighbors $K$. iFID is robust to these choices and consistently shows a strong correlation with gFID.}
\label{tab:abl}
\centering
\resizebox{0.9\linewidth}{!}{
\begin{tabular}{@{}lccccccc@{}}
\toprule
\multirow{3}{*}{Setting} & \multirow{3}{*}{Interpolation} & \multirow{3}{*}{\# Images} & \multirow{3}{*}{$K$ for $\mathrm{NN}(\cdot)$} &  \multicolumn{4}{c}{gFID SiT-XL} \\ \cmidrule(lr){5-8} 
 & & &  &  \multicolumn{2}{c}{w/o CFG} & \multicolumn{2}{c}{w/ CFG} \\ \cmidrule(lr){5-6} \cmidrule(lr){7-8} 
 &  & & & \multicolumn{1}{c}{PCC$\uparrow$} & \multicolumn{1}{c}{SRCC$\uparrow$} & \multicolumn{1}{c}{PCC$\uparrow$} & \multicolumn{1}{c}{SRCC$\uparrow$} \\ \midrule
(a) & Linear & 50k & 1 & 0.78 & 0.78 & 0.80 & 0.82 \\
(b) & Mask & 50k & 1 & 0.74 & 0.77 & 0.75 & 0.76 \\
(c) & Spherical & 50k & 1 & 0.84 & 0.85 & 0.86 & 0.81 \\
(d) & Spherical & 200k & 1 & 0.88 & 0.89 & 0.86 & 0.87 \\
(e) & Spherical & 1000k & 1 & 0.89 & 0.89 & 0.86 & 0.87 \\
(f) (default) & Spherical & 1000k & 10 & 0.89 & 0.91 & 0.88 & 0.92 \\
\bottomrule
\end{tabular}
}
\end{table}
\begin{table}[thb]
\caption{Sensitivity analysis of the iFID interpolation strength $\alpha$. As $\alpha$ increases from $0$ to $0.5$, iFID becomes more correlated with gFID and less correlated with rFID. The correlation with gFID becomes reasonably strong when $\alpha\ge 0.2$.}
\label{tab:is}
\centering
\resizebox{0.75\linewidth}{!}{
\begin{tabular}{@{}lcccccc@{}}
\toprule
\multirow{3}{*}{Interpolation strength $\alpha$} & \multicolumn{2}{c}{\multirow{2}{*}{rFID}} & \multicolumn{4}{c}{gFID SiT-XL} \\\cmidrule(lr){4-7} 
 & \multicolumn{2}{c}{} & \multicolumn{2}{c}{w/o CFG} & \multicolumn{2}{c}{w/ CFG} \\ \cmidrule(lr){2-3}\cmidrule(lr){4-5}\cmidrule(lr){6-7}   
 & PCC$\uparrow$ & SRCC$\uparrow$ & PCC$\uparrow$ & SRCC$\uparrow$ & PCC$\uparrow$ & SRCC$\uparrow$ \\ \midrule
$\alpha=0.0$ (rFID) & 1.00 & 1.00 & -0.06 & -0.21 & -0.15 & -0.31 \\
$\alpha=0.1$ & 0.21 & 0.67 & 0.37 & 0.22 & 0.38 & 0.36 \\
$\alpha=0.2$ & -0.14 & 0.14 & 0.63 & 0.53 & 0.62 & 0.52 \\
$\alpha=0.3$ & -0.14 & 0.11 & 0.66 & 0.59 & 0.65 & 0.58 \\
$\alpha=0.4$ & 0.00 & 0.19 & 0.78 & 0.79 & 0.76 & 0.79 \\
$\alpha=0.5$ (iFID default) & 0.06 & 0.10 & 0.89 & 0.91 & 0.88 & 0.92 \\ \bottomrule
\end{tabular}
}
\end{table}
\subsection{Sensitivity Analysis}
We examine several design choices in iFID to evaluate its robustness to different parameter settings.

\textbf{Interpolation method and strength.}
The choice of interpolation between $z^{(i)}$ and its nearest neighbor $z^{(i^*)}$ is important. The simplest choice is linear interpolation. For Gaussian variational autoencoders, spherical interpolation is often considered preferable \citep{Jang2024SphericalLI}. In addition, prior work on diffusion generalization \citep{Kamb2024AnAT} suggests that random mask interpolation is also a reasonable alternative. As shown in Table~\ref{tab:abl}, spherical interpolation achieves the highest correlation with gFID. Nevertheless, linear and mask interpolation also achieve strong correlations, with values around $0.8$.

By default, we set the interpolation strength of iFID to $\alpha=0.5$, which corresponds to the midpoint interpolation and, according to the ridge-set analysis, the location where the nontrivial eigenvalue along the ridge line is maximized. When $\alpha=0$, the interpolated latent is identical to the source latent $z$, and iFID reduces to rFID. Conversely, when $\alpha=1$, the interpolated latent is identical to $\mathrm{NN}(z)$. As shown in Table~\ref{tab:is}, as $\alpha$ increases from $0$ to $0.5$, the correlation between iFID and rFID decreases, while its correlation with gFID increases rapidly. When $\alpha \ge 0.2$, the correlation between iFID and gFID becomes reasonably high ($\ge 0.6$).

\textbf{Reference dataset size and Top-$K$ nearest neighbors..}
The number of images used for nearest-neighbor search may influence the results. In practice, we use the 50k ImageNet validation set as the source set for $z^{(i)}$ and the 1000k ImageNet training set as the reference set for computing $\mathrm{NN}(\cdot)$. However, Table~\ref{tab:abl} shows that reducing the number of images used for nearest-neighbor search does not significantly affect the final results. In fact, using a 50k subset of the ImageNet training set already yields a satisfactory correlation, approximately $0.85$.

We also examine whether using the top-$K$ nearest neighbors affects the results. As shown in Table~\ref{tab:abl}, replacing the single nearest neighbor with the top-$K$ nearest neighbors, with $K=10$, has minimal impact on the correlation: PCC remains $0.89$ without CFG and increases slightly from $0.86$ to $0.88$ with CFG. For this variant, we construct the interpolated latent by randomly selecting one latent from the top-$K$ nearest neighbors.

\section{Related Work}
\textbf{Variational autoencoders for diffusion models.}
Early latent diffusion models typically train VAEs solely for reconstruction \citep{rombach2022high,esser2024scaling}. Later studies show that standard reconstruction-based optimization does not necessarily improve diffusion sampling, a phenomenon known as the \say{reconstruction--generation dilemma}. Several works propose training VAEs with a diffusion prior to improve diffusion models \citep{wehenkel2021diffusion,vahdat2021score,Heek2026UnifiedL}. Other works explain and regularize VAEs from a signal-processing perspective \citep{Kouzelis2025EQVAEER,Skorokhodov2025ImprovingTD,liu2025delving,fan2025prism}. Additionally, many recent methods connect VAEs to contrastive-learning-based image encoders, such as DINO and MAE \citep{oquab2023dinov2,he2022masked}, to improve their suitability for diffusion generation \citep{chen2025softvq,Chen2025MaskedAA,zheng2025diffusion,Yao2025ReconstructionVG,ye2025distribution,yang2025latent,yao2025towards,Shi2025LatentDM,Gao2025OneLI}. Despite these advances, effective metrics for measuring VAE diffusability, as well as explanations for why the \say{reconstruction--generation dilemma} occurs, remain underexplored.

\textbf{How diffusion models generate unseen samples.}
Understanding how diffusion models generate unseen samples is a fundamental question. In general, underfitting, under-parameterization, and inductive biases in the score estimator are believed to induce deviations of the learned score from the empirical score \citep{Scarvelis2023ClosedFormDM,Kadkhodaie2023GeneralizationID,Kamb2024AnAT,Bonnaire2025WhyTraining,bonnaire2025diffusion,Song2025SelectiveUI}. These deviations prevent diffusion models from simply replicating the training data. Several works argue that diffusion models generate novel samples by composing and interpolating images from the training dataset \citep{Kamb2024AnAT,niedoba2024towards,Okawa2023CompositionalAE,Deschenaux2024GoingBC}. Similarly, \citep{Zhang2025GeneralizationOD} demonstrate that diffusion models generalize through linear interpolation in the latent space of the score-matching network. Conversely, other works show that diffusion models can generate invalid samples, or hallucinate, by interpolating between nearby modes of the training data \citep{Aithal2024UnderstandingHI,ChandranC2025LaplacianSS,baptista2025memorization}.

\subsection{Discussion and Conclusion}

Although iFID is strongly correlated with the gFID of diffusion models, there is no straightforward way to directly minimize it. One possible direction for optimizing iFID is to incorporate manifold sharpness \citep{Jeon2024UnderstandingAM}, which we leave for future work.

In conclusion, we propose interpolated FID (iFID), a simple variant of rFID that is strongly correlated with diffusion generation FID. Our key findings are as follows. First, iFID provides a useful proxy for gFID because diffusion models often generate unseen samples by interpolating training images, and iFID evaluates the quality of decoded latent interpolations. Second, reconstruction metrics can be negatively correlated with generation quality because reconstruction favors separated or disconnected latent spaces, whereas diffusion generation benefits from interpolatable latent spaces. Theoretically, iFID evaluates decoded interpolations that are aligned with an approximate ridge-set manifold around which diffusion samples concentrate. To our knowledge, iFID is the first metric shown to strongly correlate with diffusion gFID across diverse VAEs.

\newpage
\bibliography{iclr2026_conference}
\bibliographystyle{plain}

\newpage
\appendix
\section{Proof of Main Results}
\label{app:pf}
\textbf{Theorem 3. } (The Jacobian is close to the two-point Jacobian when $t$ is small) \textit{
When $t$ is small, the full Jacobian is close to the Jacobian determined by
$z^{(i^*)}$ and $z^{(j^*)}$, the two nearest-neighbor training samples to $z_t$.
Assume that $B=\max_k\|z^{(k)}\| < \infty$ and $\Delta = \min_{k\neq i^*,j^*} ||z_t-\alpha_tz^{(k)}||^2 - ||z_t-\alpha_tz^{(j^*)}||^2 > 0$. Then,
\begin{align}
    \|J-J_2\| &\le \frac{\alpha_t^2}{\sigma_t^4}
    \left(9B^2(M-2)\right)\exp{\left(-\frac{\Delta}{2\sigma_t^2}\right)},\notag \\ 
    \textup{where } J_2 &:= -\frac{1}{\sigma_t^2}I
    + \frac{\alpha_t^2}{\sigma_t^4}\sum_{k=i^*,j^*}\tilde{w}^{(k)}
    (z^{(k)}-\tilde{\mu})(z^{(k)}-\tilde{\mu})^T.
\end{align}
Here, $J_2$ is the Jacobian of the optimal score for the two-point training set $\{z^{(i^*)},z^{(j^*)}\}$, where 
\begin{gather}
    \tilde{w}^{(k)} = \frac{w^{(k)}}{\sum_{l=i^*,j^*}w^{(l)}},\tilde{\mu}=\sum_{k=i^*,j^*}\tilde{w}^{(k)}z^{(k)}. \notag
\end{gather}
}
\begin{proof}
    Denote 
    \begin{align}
        d_k&=||z_t-\alpha_tz^{(k)}||^2,  \notag \\
        w^{(k)} &= \frac{\exp{(-d_k/2\sigma_t^2)}}{\sum_{l}\exp{(-d_l/2\sigma_t^2)}},\notag \\
        \tilde{w}^{(i^*)} &= \frac{\exp{(-d_{i^*}/2\sigma_t^2)}}{\exp{(-d_{i^*}/2\sigma_t^2)}+\exp{(-d_{j^*}/2\sigma_t^2)}}\notag \\
        &= \frac{w^{(i^*)}}{w^{(i^*)}+w^{(j^*)}},\notag\\
        \tilde{w}^{(j^*)} &= 1 - \tilde{w}^{(i^*)}\notag\\
        &= \frac{w^{(j^*)}}{w^{(i^*)}+w^{(j^*)}},\notag\\
        \mu &= \sum_{k=1}^M w^{(k)}z^{(k)}, \notag\\
        \tilde{\mu} &= \tilde{w}^{(i)}z^{(i)} + \tilde{w}^{(j)}z^{(j)}.
    \end{align}
    Without loss of generality, we assume
    \begin{align}
        d_{j^*} \ge d_{i^*}.
    \end{align}
    And we denote the distance of third closest $z^{(k)}$ as 
    \begin{align}
    \Delta &= \min_{k\neq i,j}(d_k - d_{j^*}) > 0,\notag\\
    \delta &= z^{(i^*)} - z^{(j^*)}.
    \end{align}    
    Denote the full covariance, a subpart of covariance and the covariance of two point as:
    \begin{align}
        \mathcal{C} &= \sum_{k}w^{(k)}(z^{(k)} - \mu)(z^{(k)} - \mu)^T, \notag \\
        \mathcal{C}_{*} &= w^{(i^*)}(z^{(i^*)} - \mu)(z^{(i^*)}-\mu)^T+w^{(j^*)}(z^{(j^*)}-\mu)(z^{(j^*)} - \mu)^T,\notag \\
        \mathcal{C}_2 &=\tilde{w}^{(i^*)}(z^{(i^*)} - \tilde{\mu})(z^{(i^*)}-\tilde{\mu})^T+\tilde{w}^{(j^*)}(z^{(j^*)}-\tilde{\mu})(z^{(j^*)} - \tilde{\mu})^T \notag \\&= \tilde{w}^{(i)}\tilde{w}^{(j)}\delta\delta^T.
    \end{align}
    Then the jacobian error
    \begin{align}        
        J - J_2 &= \frac{\alpha_t^2}{\sigma_t^4}(\mathcal{C} - \mathcal{C}_2) \notag \\
        &= \frac{\alpha_t^2}{\sigma_t^4}(\underbrace{\mathcal{C}_{*} - \mathcal{C}_2}_{\mathcal{E}_1\textup{ , pair error}} + \underbrace{\sum_{k\neq i,j}w^{(k)}(z^{(k)}-\mu)(z^{(k)}-\mu)^T}_{\mathcal{E}_2\textup{, tail error}}).
    \end{align}
    We first consider the tail error $\mathcal{E}_2$.
    \begin{itemize}
        \item  We bound the normalization constant of softmax terms: 
    \begin{gather}
        \sum_l^M w^{(l)} = \exp{(-d_{i^*}/2\sigma_t^2)}(1+\exp{(-(d_{j^*}-d_{i^*})/2\sigma_t^2)}+\sum_{k\neq i,j^M} \exp{(-(d_k-d_{i^*})/2\sigma_t^2)}. \label{eq:wsum}
    \end{gather}
    By definition, $\forall k \neq i^*,j^*$,
    \begin{gather}
        d_k - d_{i^*} \ge (d_{j^*}-d_{i^*}) + \Delta. \label{eq:dki}
    \end{gather}
    Therefore
    \begin{align}
        \exp{(-(d_k-d_{i^*})/2\sigma_t^2)} &\overset{\textup{Eq.}\ref{eq:dki}}{\le} \exp{(-(d_{j^*}-d_{i^*})/2\sigma_t^2)}\exp{(-\Delta/2\sigma_t^2)},\notag\\
        \sum_{k\neq i^*,j^*} \exp{(-(d_k-d_{i^*})/2\sigma_t^2)} &\le (M-2) \exp{(-(d_{j^*}-d_{i^*})/2\sigma_t^2)}\exp{(-\Delta/2\sigma_t^2)}.
    \end{align}
    And therefore $\forall k \neq i^*,j^*$, we have
    \begin{align} 
        w^{(k)} &= \frac{\exp{(-d_k/2\sigma_t^2)}}{\sum_l^M w^{(l)}} \notag \\&\overset{\textup{Eq.}\ref{eq:wsum}}{\le} \frac{\exp{(-(d_k-d_{i^*})/2\sigma_t^2)}}{1} \notag \\
        &\overset{\textup{Eq.}\ref{eq:dki}}{\le} \exp{(-(d_{j^*}-d_{i^*})/2\sigma_t^2)}\exp{(-\Delta/2\sigma_t^2)} \notag \\
        &\le \exp{(-\Delta/2\sigma_t^2)},\\
        \sum_{k\neq i^*,j^*} w^{(k)} &\le (M-2)\exp{(-\Delta/2\sigma_t^2)}.\label{eq:wnij}
    \end{align}
    \item Next, notice that
    \begin{gather}
        ||z^{(k)} - \mu|| \le ||z^{(k)}|| + ||\mu|| \le 2B.
    \end{gather}
    Combine with Eq~\ref{eq:wnij}, we have
    \begin{gather}
        \mathcal{E}_2\le4B^2(M-2)\exp{(-\Delta/2\sigma_t^2)}
    \end{gather}
    \end{itemize}
Next, we consider the pair error $\mathcal{E}_1$. 
\begin{itemize}
    \item First, we rewrite $\mathcal{C_*}$. We define mean shift $r$, which is the difference between the posterior mean of $M$ samples and $2$ samples as
    \begin{align}
        r=\mu - \tilde{\mu}.
    \end{align}
    Also, we define the distance to mean as
    \begin{align}
        a^{(i^*)} &= z^{(i^*)} - \tilde{\mu}, \notag \\
        a^{(j^*)} &= z^{(j^*)} - \tilde{\mu}.
    \end{align}
    Then
    \begin{align}
        z^{(i^*)} - \mu &= (z^{(i^*)} - \tilde{\mu}) - (\mu - \tilde{\mu})\notag \\
        & = a^{(i^*)} - r, \notag \\
    z^{(j^*)} - \mu &= a^{(j^*)} - r.
    \end{align}
    Therefore,
    \begin{align}
        \mathcal{C}_{*} =& w^{(i^*)}(a^{(i^*)} - r)(a^{(i^*)} - r)^T+w^{(j^*)}(a^{(j^*)} - r)(a^{(j^*)} - r)^T \notag \\
    =&w^{(i^*)}a^{(i^*)}a^{(i^*)T}+w^{(j^*)}a^{(j^*)}a^{(j^*)T}\notag \\&-(w^{(i^*)}a^{(i^*)}+w^{(j^*)}a^{(j^*)})r^T-r(w^{(i^*)}a^{(i^*)}+w^{(j^*)}a^{(j^*)})^T\notag \\&+(w^{(i^*)}+w^{(j^*)})rr^T.
    \end{align}
    Because
    \begin{align}
        w^{(i^*)}a^{(i^*)}+w^{(j^*)}a^{(j^*)} &= w^{(i^*)}(z^{(i^*)}-\tilde{\mu})+w^{(j^*)}(z^{(j^*)}-\tilde{\mu})\notag \\
        &= (w^{(i^*)}z^{(i^*)}+w^{(j^*)}z^{(j^*)}) - (w^{(i^*)}+w^{(j^*)})\tilde{\mu} \notag \\
        &=(w^{(i^*)}z^{(i^*)}+w^{(j^*)}z^{(j^*)}) - (w^{(i^*)}+w^{(j^*)})\frac{(w^{(i^*)}z^{(i^*)}+w^{(j^*)}z^{(j^*)})}{w^{(i^*)}+w^{(j^*)}} \notag \\
        &=0.
    \end{align}
    Therefore,
    \begin{align}
        \mathcal{C}_{*} = w^{(i^*)}a^{(i^*)}a^{(i^*)T}+w^{(j^*)}a^{(j^*)}a^{(j^*)T}+(w^{(i^*)}+w^{(j^*)})rr^T.
    \end{align}
    Next, we note that
    \begin{align}
        a^{(i^*)} =& z^{(i^*)} - \tilde{\mu} \notag \\
         =&\tilde{w}^{(j^*)}\delta,\notag \\
    a^{(j^*)} =& -\tilde{w}^{(i^*)}\delta.
    \end{align}
    Therefore,
    \begin{align}
        w^{(i^*)}a^{(i^*)}a^{(i^*)T} =& (w^{(i^*)}+w^{(j^*)})\tilde{w}^{(i^*)}\tilde{w}^{(j^*)2}\delta\delta^T,\notag \\
        w^{(j^*)}a^{(j^*)}a^{(j^*)T}=&(w^{(i^*)}+w^{(j^*)})\tilde{w}^{(j^*)}\tilde{w}^{(i^*)2}\delta\delta^T,\notag \\
        w^{(i^*)}a^{(i^*)}a^{(i^*)T}+w^{(j^*)}a^{(j^*)}a^{(j^*)T}=&(w^{(i^*)}+w^{(j^*)})(\tilde{w}^{(i^*)}+\tilde{w}^{(j^*)})\tilde{w}^{(i^*)}\tilde{w}^{(j^*)}\delta\delta^T\notag \\
        =& (w^{(i^*)}+w^{(j^*)})\tilde{w}^{(i^*)}\tilde{w}^{(j^*)}\delta\delta^T.
    \end{align}
    Thus, we have
    \begin{align}
        \mathcal{C}_*=(w^{(i^*)}+w^{(j^*)})(\tilde{w}^{(i^*)}\tilde{w}^{(j^*)}\delta\delta^T+rr^T)\notag \\
        = (w^{(i^*)}+w^{(j^*)})(C_2+rr^T).
    \end{align}
    \item Next, we consider the difference $\mathcal{C_*} - \mathcal{C_2}$. We have
    \begin{align}
        \mathcal{C_*} - \mathcal{C_2} = \sum_{k\neq i^*,j^*}w^{(k)}C_2+(w^{(i^*)}+w^{(j^*)})rr^T.
    \end{align}
    The first term, we have
    \begin{align}
    ||\sum_{k\neq i^*,j^*}w^{(k)}\mathcal{C}_2||\le&||\sum_{k\neq i^*,j^*}w^{(k)}||||w^{(i^*)}w^{(j^*)}\delta\delta^T||\notag \\
        \le&||\sum_{k\neq i^*,j^*}w^{(k)}||||\frac{1}{4}\delta\delta^T||\notag \\
        \le&||\sum_{k\neq i^*,j^*}w^{(k)}||\frac{1}{4}4B^2\notag \\
        =& ||\sum_{k\neq i^*,j^*}w^{(k)}||B^2,\notag \\
        \le&B^2(M-2)\exp{(-\Delta/2\sigma_t^2)}.
    \end{align}
    For the second term, we rewrite mean shift term
    \begin{align}
        r=&\mu - \tilde{\mu}\notag \\
        =&w^{(i^*)}z^{(i^*)}+w^{(j^*)}z^{(i^*)}+\sum_{k\neq i^*,j^*}w^{(k)}z^{(k)} - (\tilde{w}^{(i^*)}z^{(i^*)}+\tilde{w}^{(j^*)}z^{(i^*)})\notag \\
        =&(w^{(i^*)}+w^{(j^*)})\tilde{\mu}+\sum_{k\neq i^*,j^*}w^{(k)}z^{(k)} - (\tilde{w}^{(i^*)}z^{(i^*)}+\tilde{w}^{(j^*)}z^{(i^*)})\notag \\
        =&(w^{(i^*)}+w^{(j^*)}-1)\tilde{\mu}+\sum_{k\neq i^*,j^*}w^{(k)}z^{(k)} \notag \\
        =&-\sum_{k\neq i^*,j^*}w^{(k)}\tilde{\mu}+\sum_{k\neq i^*,j^*}w^{(k)}z^{(k)} \notag  \notag \\
        =& \sum_{k\neq i^*,j^*}w^{(k)}(z^{(k)}-\tilde{\mu}).
    \end{align}
    Therefore, the second term can be bounded as
    \begin{align}
        ||(w^{(i^*)}+w^{(j^*)})rr^T||\le& ||(w^{(i^*)}+w^{(j^*)})||||r||^2\notag \\
        \le & ||r||^2\notag \\
        \le &||\sum_{k\neq i^*,j^*}w^{(k)}(z^{(k)}-\tilde{\mu})||^2\notag \\
        \le & ||\sum_{k\neq i^*,j^*}w^{(k)}||^2||(z^{(k)}-\tilde{\mu})||^2 \notag \\
        \le & ||\sum_{k\neq i^*,j^*}w^{(k)}||||(z^{(k)}-\tilde{\mu})||^2\notag\\
        \le & 4B^2||\sum_{k\neq i^*,j^*}w^{(k)}||\notag \\
        \le & 4B^2(M-2)\exp{(-\Delta/2\sigma_t^2)}.
    \end{align}
    \item Combining those two terms, we have
    \begin{align}
        ||\mathcal{E}_1|| =& ||\mathcal{C_*} - \mathcal{C}_2||\notag \\
        \le& B^2(M-2)\exp{(-\Delta/2\sigma_t^2)}+4B^2(M-2)\exp{(-\Delta/2\sigma_t^2)} \notag \\
        =& 5B^2(M-2)\exp{(-\Delta/2\sigma_t^2)}.
    \end{align}
\end{itemize}
Combining the bounds on $\mathcal{E}_1,\mathcal{E}_2$, we have
\begin{align}
    ||J-J_2|| \le \frac{\alpha_t^2}{\sigma_t^4}9B^2(M-2)\exp{(-\frac{\Delta}{2\sigma_t^2})}.
\end{align}
Clearly, when $t$ is small, $\alpha_t \rightarrow 1, \sigma_t\rightarrow 0$, and
\begin{align}
    \lim_{t\rightarrow0}||J-J_2||= 0.
\end{align}
Because 
\begin{align}
    \lim_{\sigma_t\rightarrow0} \frac{1}{\sigma_t^4}\exp{(-\frac{\Delta}{2\sigma_t^2})} = 0.
\end{align}
\end{proof}
\textbf{Theorem 4.} (Diffusion samples concentrate around the ridge set determined by the Jacobian {\citep{He2026DiffusionMG}}) \textit{
Diffusion samples $z_t$ concentrate around the ridge set
\begin{gather}
    \mathcal{R}\left(\beta=\frac{1}{\sigma_t^2},t\right)
    =
    \left\{z_t \mid z_t = (1-\gamma)z^{(i^*)} + \gamma z^{(j^*)},\ \gamma \in \mathbb{R}\right\}.
\end{gather}
Moreover, the largest eigenvalue is attained at the midpoint, when $\gamma=0.5$ and $z_t=(z^{(i^*)}+z^{(j^*)})/2$.
}
\begin{proof}
    Notice that
    \begin{align}
        J_2 &= -\frac{1}{\sigma_t^2}+\frac{\alpha_t^2}{\sigma_t^4}\sum_{k=i^*,j^*}\tilde{w}^{(k)}(z^{(k)}-\tilde{\mu})(z^{(k)}-\tilde{\mu})^T\notag \\
        &= -\frac{1}{\sigma_t^2}+\frac{\alpha_t^2}{\sigma_t^4}\tilde{w}^{(i^*)}\tilde{w}^{(j^*)}\delta\delta^T.
    \end{align}
    The top eigenvalue and eigenvector is the direction alone the connection line of $z^{(i^*)},z^{(j^*)}$:
    \begin{gather}
        \lambda_1= - \frac{1}{\sigma_t^2}+\frac{\alpha_t^2}{\sigma_t^4}\tilde{w}^{(i^*)}\tilde{w}^{(j^*)}||\delta_{ij}||, 
        u_1= \frac{\delta_{ij}}{||\delta_{ij}||}.
    \end{gather}
    And the other eigenvalues are
    \begin{gather}
        \lambda_{2,...,d} = -\frac{1}{\sigma_t^2}.
    \end{gather}
    When $\beta=\frac{1}{\sigma_t^2}$, all other eigenvalues are discarded and only one dimension is left for ridge, as $\tilde{w}^{(i^*)}\tilde{w}^{(j^*)} \ge 0$. Therefore, for 1-dimensional ridge, we must have
    \begin{gather}
        s_*(z_t,t) \in \textup{span}(\delta).
    \end{gather}
    On the other hand, the empirical score in this case is
    \begin{gather}
        s_*(z_t,t) = \frac{1}{\sigma_t^2}(-z_t+\alpha_t\tilde{\mu}).
    \end{gather}
    Notice that 
    \begin{gather}
        \tilde{\mu} = \tilde{w}^{(i^*)}z^{(i^*)}+(1-\tilde{w}^{(j^*)}))z^{(j^*)},
    \end{gather}
    which means $\alpha_t\tilde{\mu}$ pass through $\alpha_tz^{(i^*)}$ and $\alpha_tz^{(j^*)}$.
    Besides, we need to have
    \begin{gather}
        -z_t+\alpha_t\tilde{\mu} \in \textup{span}(\delta),\notag \\
    \end{gather}
    However, we already know $\tilde{\mu}$ pass through $\alpha_tz^{(i^*)}$ and $\alpha_tz^{(j^*)}$. When $t$ is small, $\alpha_t=1$. This means that the ridge set is just the line passing through $z^{(i^*)},z^{(j^*)}$.
    On the other hand, when $\gamma=0.5$, clearly $\lambda_1$ is maximized with $\tilde{w}^{(i^*)} = \tilde{w}^{(j^*)} = 0.5$.
\end{proof}
\section{Additional Experimental Results}
\subsection{Pre-Trained VAEs}

\textbf{SD-VAE} The Stable Diffusion VAE (SD-VAE) \cite{rombach2022high}, is among the first VAEs to be adopted in large-scale latent diffusion models. It is trained with a mixture of KL-divergence penalties, mean-squared error, LPIPS, and adversarial loss. It has a 2D latent with shape $4\times 32\times 32$. We adopt the pre-trained checkpoint from \url{https://huggingface.co/stable-diffusion-v1-5/stable-diffusion-v1-5/tree/main/vae}

\textbf{IN-VAE} ImageNet VAE (IN-VAE) \cite{Yao2025ReconstructionVG} serves as the baseline VAE for REPA-E \cite{Leng2025REPAEUV}. It is trained with a mixture of KL divergence penalties, mean-squared error, LPIPS, and adversarial loss. It uses a 2D latent representation with shape $32\times 16\times 16$. We adopt the pre-trained checkpoint from \url{https://huggingface.co/REPA-E/invae/tree/main}.

\textbf{FLUX-VAE} FLUX-VAE \cite{flux2024} is the VAE for FLUX-1 models. It extends the channel dimension of SD-VAE to 32 and achieves higher reconstruction performance. It has a 2D latent representation with shape $16\times 32\times 32$. We adopt the pre-trained checkpoint from \url{https://huggingface.co/black-forest-labs/FLUX.1-dev/tree/main/vae}.

\textbf{QW-VAE} Qwen-Image VAE (QW-VAE) \cite{wu2025qwen} is the VAE for the Qwen-Image Model. It uses the same latent shape as FLUX-VAE, namely a 2D latent with shape $16\times 32\times 32$. We adopt the pre-trained checkpoint from \url{https://huggingface.co/Qwen/Qwen-Image/tree/main/vae}.

\textbf{SD3-VAE} Stable Diffusion 3 VAE (SD3-VAE) \cite{esser2024scaling} is the VAE used in Stable Diffusion 3 models. It has the same latent shape as FLUX-VAE. It uses a 2D latent with shape $16\times 32\times 32$. We adopt the pre-trained checkpoint at \url{https://huggingface.co/black-forest-labs/FLUX.1-dev/tree/main/vae}.

\textbf{EQ-VAE} Equivariance VAE (EQ-VAE) \cite{Kouzelis2025EQVAEER} is proposed as an improvement for diffusion training. In addition to the SD-VAE loss, EQ-VAE is regularized with an equivariance constraint that ensures that the affine transform applied to the latent space has the same effect as the affine transform applied to the image. It has a 2D latent with shape $4\times 32\times 32$. We adopt the pre-trained checkpoint at \url{https://huggingface.co/zelaki/eq-vae/tree/main}.

\textbf{VA-VAE} Vision foundation model Aligned Variational AutoEncoder (VA-VAE) \cite{Yao2025ReconstructionVG} is proposed to improve Gaussian VAE for diffusion training. In addition to the standard VAE loss, VA-VAE additionally adopts vision foundation models such as DINO \cite{oquab2023dinov2} and aligns the latent space linearly to the feature encoded by DINOv2. Such alignment empirically improves diffusion sample quality. We evaluate two variants of VA-VAE, with 2D latents of shape $32\times 16\times 16$ and $64\times 16\times 16$, respectively. We adopt the pre-trained checkpoint in \url{https://huggingface.co/hustvl/va-vae-imagenet256-experimental-variants/tree/main}.

\textbf{SOFT-VQ} Soft VQ-VAE (SOFT-VQ) \cite{chen2025softvq} is a variational autoencoder with a 1D latent space and a transformer architecture. It also adopts vision foundation models, such as DINOv2, for alignment. It has a 1D latent with shape $64\times 32$. We adopt the pre-trained checkpoint from \url{https://huggingface.co/SoftVQVAE/softvq-l-64}.

\textbf{MAE-TOK} Masked Autoencoders (MAE-TOK) \cite{Chen2025MaskedAA} are an improvement over SOFT-VQ, which additionally introduces mask training following MAE \cite{he2022masked} and introduces auxiliary decoder loss functions aligned with DINOv2 features. It has a 1D latent with shape $128\times 32$. We adopt the pre-trained checkpoint from \url{https://huggingface.co/MAETok/maetok-b-128}.

\textbf{DE-TOK} Denoising Tokenizer (DE-TOK) is a combination of a denoising autoencoder \cite{Vincent2011ACB} and MAE \cite{he2022masked}. During training, it adds noise to tokens, while using mask training at the same time. It has a 1D latent with shape $128\times 32$. We adopt the pre-trained checkpoint from \url{https://huggingface.co/jjiaweiyang/l-DeTok}.

\textbf{DM-VAE} Distribution Matching VAE (DM-VAE) adopts Distribution Matching Distillation (DMD) \cite{Yin2024ImprovedSynthesis} as the prior for the VAE. We find that alignment to the DINO prior derived from ten ImageNet classes is most beneficial for generation. It has a 1D latent with shape $256\times 32$. We adopt the pre-trained checkpoint in \url{https://huggingface.co/sen-ye/dmvae/tree/main}.

\textbf{REPAE-VAE} End-to-end representation alignment VAE (REPAE-VAE) \cite{Leng2025REPAEUV} is a by-product of optimizing a VAE with a REPA alignment \cite{oquab2023dinov2}. We adopt the REPAE-VAE that starts from SD-VAE. It has a 2D latent representation with shape $4\times 32\times 32$. We adopt the pre-trained checkpoint in \url{https://huggingface.co/REPA-E/e2e-sdvae}.

\textbf{RAE} Representation-aligned autoencoder (RAE) \cite{Yu2024RepresentationAF} removes the reconstruction loss in the encoding stage. It directly adopts DINOv2 \cite{oquab2023dinov2} as the image encoder and trains a decoder for reconstruction. It has a 2D latent representation with shape $768\times 16\times 16$. We adopt the pre-trained checkpoint in \url{https://huggingface.co/nyu-visionx/RAE-collections/tree/main}

\textbf{Training} We train all SiT \cite{Ma2024SiTEF} on the ImageNet training split for 40 epochs, following \cite{Leng2025REPAEUV}. For VAE with 2D latents, we adopt patch sizes of 1 or 2 to consistently produce 256 tokens. We employ 1D/2D sine-cosine positional embeddings for VAEs with 1D/2D latents. We additionally apply a time shift introduced in \cite{esser2024scaling,zheng2025diffusion} to eliminate the effect of different total latent dimensions on a diffusion model.

\textbf{Sampling} We adopt the Euler SDE solver implemented in \cite{Ma2024SiTEF} with 250 steps. For results with CFG, we search the classifier-free guidance scale in the range [1.0, 6.0] with a step size of 0.25 and report the best gFID.

\subsection{Metrics}

\textbf{PSNR, SSIM, LPIPS, rFID} The peak signal-to-noise ratio (PSNR), structural similarity index measure (SSIM), and learned perceptual image patch similarity (LPIPS), together with reconstruction FID (rFID) \cite{Wang2004ImageQA, zhang2018unreasonable, Heusel2017GANsTB}, are among the most commonly used reconstruction metrics for training and evaluating VAEs. PSNR, SSIM, and LPIPS are essentially similar as they are computed between a single source image and its reconstruction. However, they are strongly negatively correlated with the quality of samples produced by diffusion models. rFID measures the divergence between the source image distribution and the target image distribution. Additionally, its correlation with gFID is weak.

\textbf{Diffusion Loss} The diffusion loss can be viewed as a form of the variational lower bound for a diffusion model \cite{song2021maximum}. Essentially, when weighted correctly, the diffusion loss bounds the divergence between the source latent distribution and the sampled latent distribution. At first glance, the diffusion loss should be most correlated with sample quality. However, as we have shown empirically, the correlation between diffusion loss and sample quality is weak. The diffusion loss primarily measures the memorization capability of a diffusion model \cite{Buchanan2025OnTE}. By contrast, what determines sample quality is the generalization ability of a diffusion model. Previous works \cite{Chen2025MaskedAA} successfully predict diffusion loss using the number of modes of the latent space as estimated by a Gaussian mixture model. In practice, however, the diffusion loss is not well correlated with sample quality.

\textbf{EQ Loss \cite{Kouzelis2025EQVAEER}} The Equivariance Loss (EQ Loss) evaluates whether a spatial transformation applied to the latent representation is preserved after decoding. Given an input image $x$ and its latent representation $z$, and letting $g(\cdot)$ denote the decoder, the transformed reconstruction term is
\begin{equation}
\mathcal{L}_{\mathrm{EQ}}(x;\tau)
=
\Delta\!\left(\tau \circ x,\; g(\tau \circ z)\right).
\end{equation}
where $\Delta(\cdot,\cdot)$ denotes the reconstruction distortion. The transformation $\tau$ is sampled from spatial transformations composed of scaling and $90^\circ$-multiple rotations. The final score is averaged over the evaluation set. Lower EQ Loss indicates better equivariance under such spatial transformations.

\textbf{VF Loss \cite{Yao2025ReconstructionVG}} The Vision Foundation model alignment loss (VF loss) aligns the VAE latent features with the features of a frozen vision foundation model. Given an input $x$, let $z$ denote the latent representation produced by the encoder and let $f$ denote the feature representation extracted by the foundation model. A linear projection $W$ is first used to match the latent dimension to the foundation-model feature dimension:
\begin{equation}
z' = W z.
\end{equation}
Let $d$ denote the number of feature positions after flattening all non-channel dimensions. We use $z'_i$ and $f_i$ to denote the feature vectors at position $i \in \{1,\dots,d\}$ in $z'$ and $f$, respectively. The VF loss consists of two terms. The first term is a marginal cosine similarity term:
\begin{equation}
\mathcal{L}_{\mathrm{mcos}}
=
\frac{1}{d}
\sum_{i=1}^{d}
\mathrm{ReLU}\!\left(
1-m_1-\frac{z'_i\cdot f_i}{\|z'_i\|\,\|f_i\|}
\right).
\end{equation}
The second term is a marginal distance matrix similarity term, where $i,j \in \{1,\dots,d\}$ index two feature positions:
\begin{equation}
\mathcal{L}_{\mathrm{mdms}}
=
\frac{1}{d^2}
\sum_{i=1}^{d}\sum_{j=1}^{d}
\mathrm{ReLU}\!\left(
\left|
\frac{z'_i\cdot z'_j}{\|z'_i\|\,\|z'_j\|}
-
\frac{f_i\cdot f_j}{\|f_i\|\,\|f_j\|}
\right|
-m_2
\right).
\end{equation}
Here, $m_1$ and $m_2$ are margin hyperparameters. We define the overall VF loss as
\begin{equation}
\mathcal{L}_{\mathrm{VF}} = \mathcal{L}_{\mathrm{mcos}} + \mathcal{L}_{\mathrm{mdms}}.
\end{equation}

\textbf{SE Loss \cite{Skorokhodov2025ImprovingTD}} The Scale-Equivariance Loss (SE Loss) regularizes reconstruction by enforcing consistency between downsampled RGB signals and reconstructions from downsampled latents. Given an input image $x$ and its latent representation $z$, let $s(\cdot)$ denote a downsampling operator, and define $\tilde{x}=s(x)$ and $\tilde{z}=s(z)$. The SE objective is
\begin{equation}
\mathcal{L}_{\mathrm{SE}}(x)
=
\Delta(x,g(z))
+
\alpha\, \Delta(\tilde{x},g(\tilde{z})),
\end{equation}
where $\Delta(\cdot,\cdot)$ denotes the reconstruction distortion and $\alpha$ is the scale-equivariance regularization weight. Lower SE Loss indicates better scale equivariance.

\begin{figure}[thb]
\centering
    \includegraphics[width=0.5\linewidth]{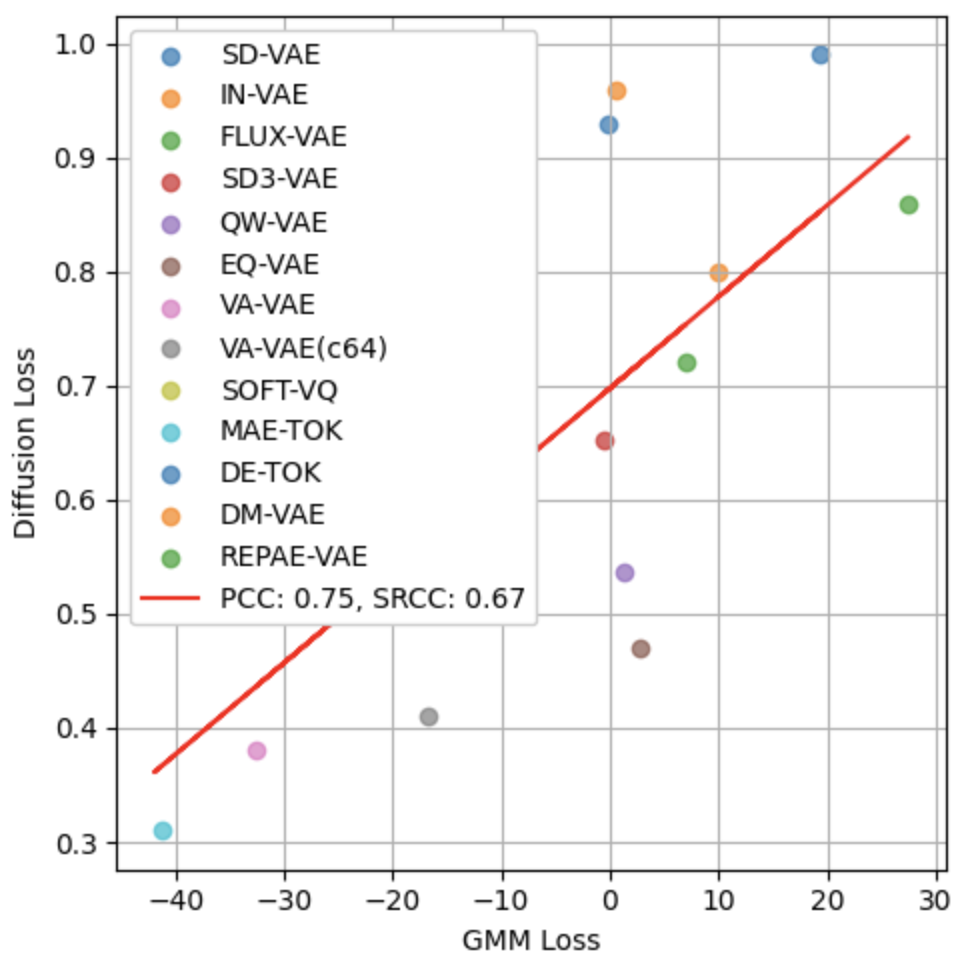}
\caption{The relationship between the GMM loss and the diffusion loss is strong, which supports the theoretical results presented in \cite{Chen2025MaskedAA}.}
\label{fig:gmm}
\end{figure}

\textbf{GMM Loss \cite{Chen2025MaskedAA}} The Gaussian Mixture Model Loss (GMM Loss) quantifies the structural complexity of the latent space by fitting a Gaussian mixture model with a fixed number of components. Given latent representations $z^{(1:N)}$, we fit a $K$-component Gaussian mixture model (GMM) to the latent space and evaluate its negative log-likelihood (NLL):
\begin{equation} \mathcal{L}_{\mathrm{GMM}}
-\frac{1}{N}\sum_{n=1}^{N}\log p_{\mathrm{GMM},K}(z^{(n)}),
\end{equation}
where $p_{\mathrm{GMM},K}$ denotes the density of the fitted $K$-component GMM. A lower GMM Loss indicates that the latent distribution can be better approximated by a $K$-component GMM, suggesting a simpler latent structure.

In \cite{Chen2025MaskedAA}, the authors theoretically demonstrate that a latent space with lower GMM Loss also exhibits lower diffusion loss, and therefore propose using GMM Loss as an indicator of diffusion sample quality. In Figure~\ref{fig:gmm}, we show that GMM Loss exhibits a strong correlation with diffusion loss. However, diffusion loss does not necessarily correlate highly with sample quality when the latent spaces differ.

\subsection{The Relationship Between Metrics and gFID}
Additionally, we provide comprehensive plots illustrating the relationship between gFID and all metrics, as shown in Figures~\ref{fig:rrec} and~\ref{fig:rnrec}.
\begin{figure}[thb]
\centering
    \includegraphics[width=\linewidth]{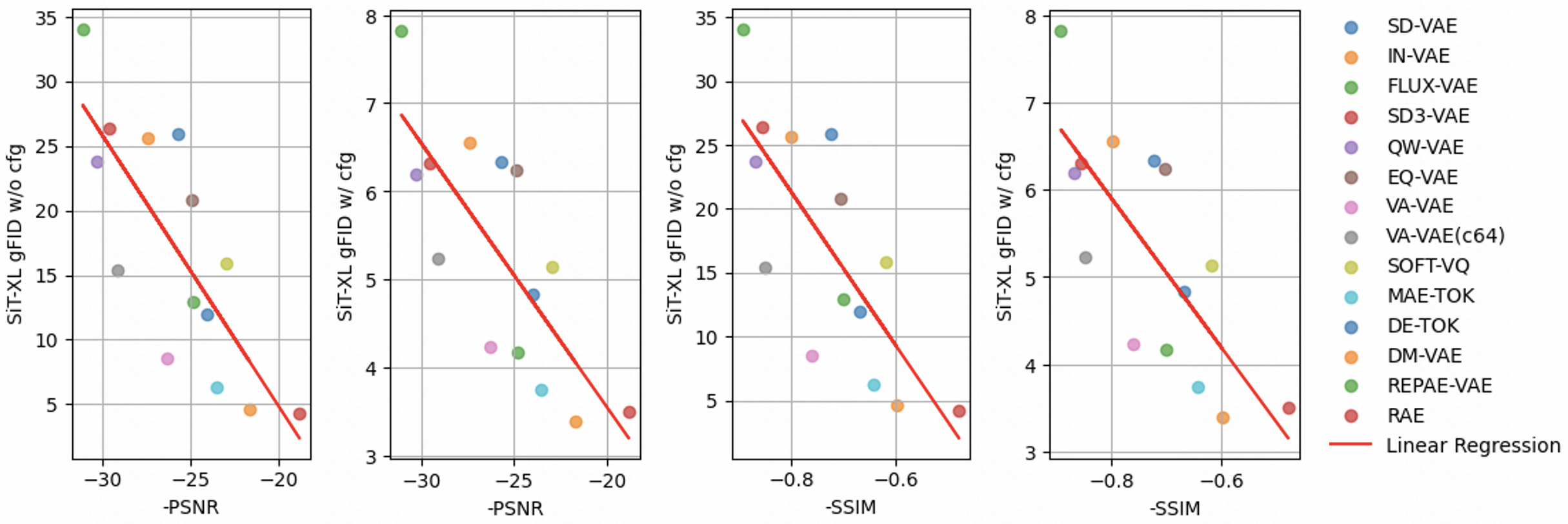}
    \includegraphics[width=\linewidth]{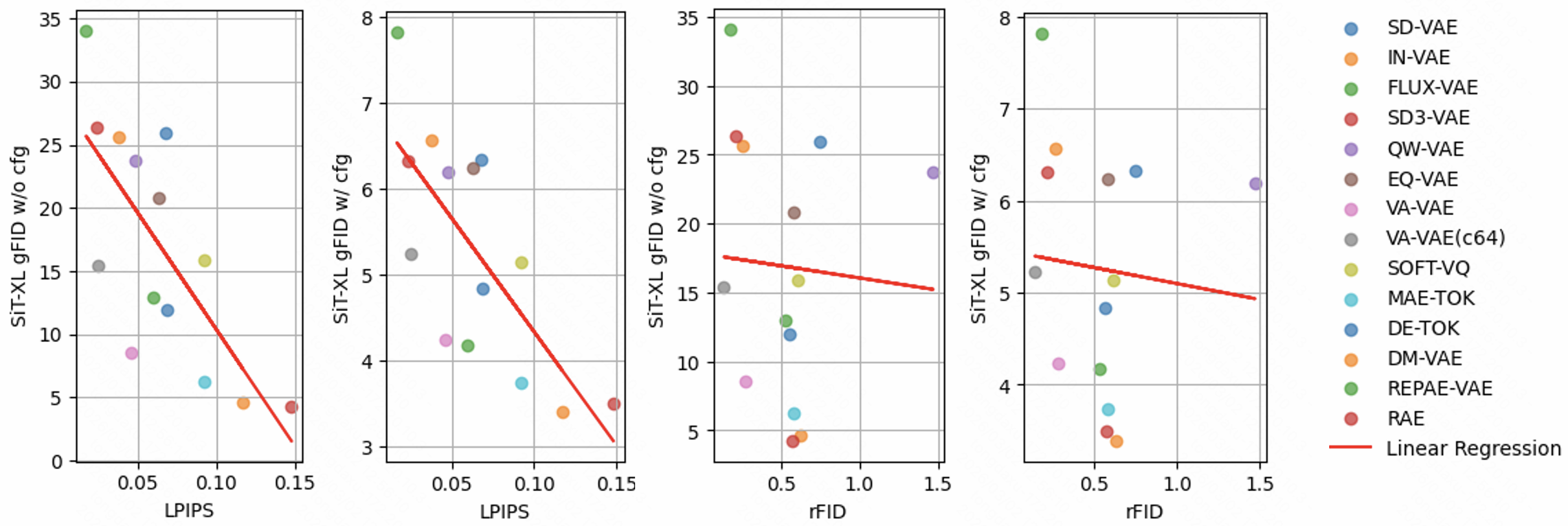}
\caption{The relationship between reconstruction metrics and gFID for SiT/XL. The results indicate that reconstruction metrics are often negatively correlated with gFID.}
\label{fig:rrec}
\end{figure}

\begin{figure}[thb]
\centering
    \includegraphics[width=\linewidth]{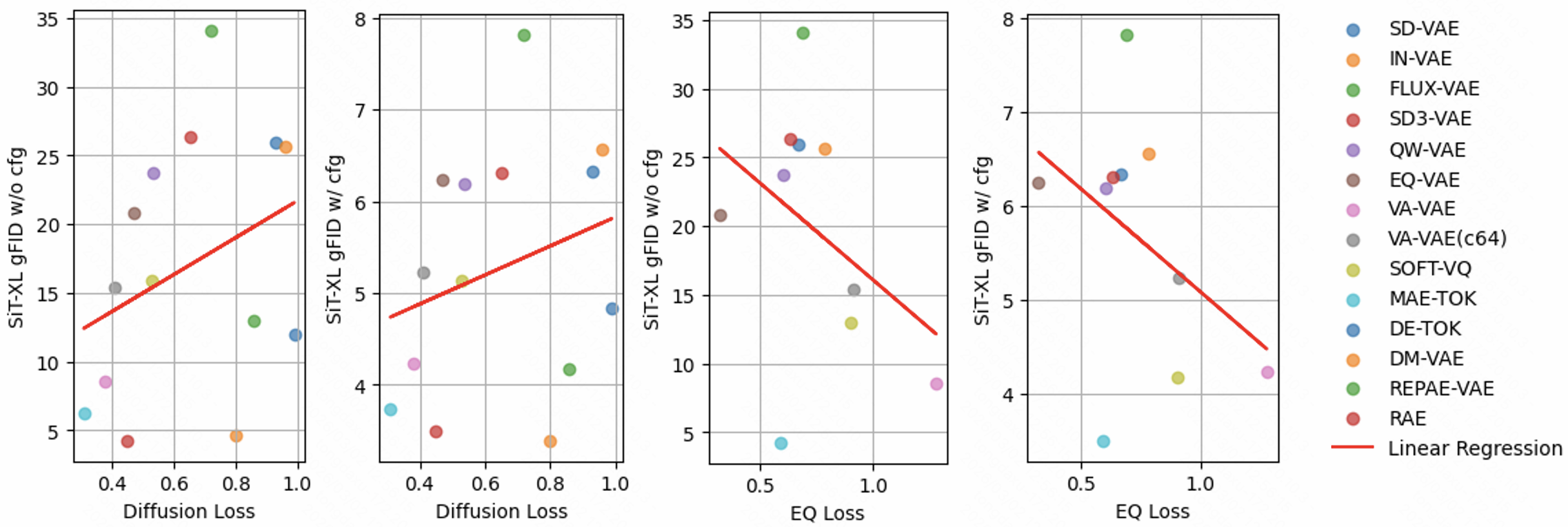}
    \includegraphics[width=\linewidth]{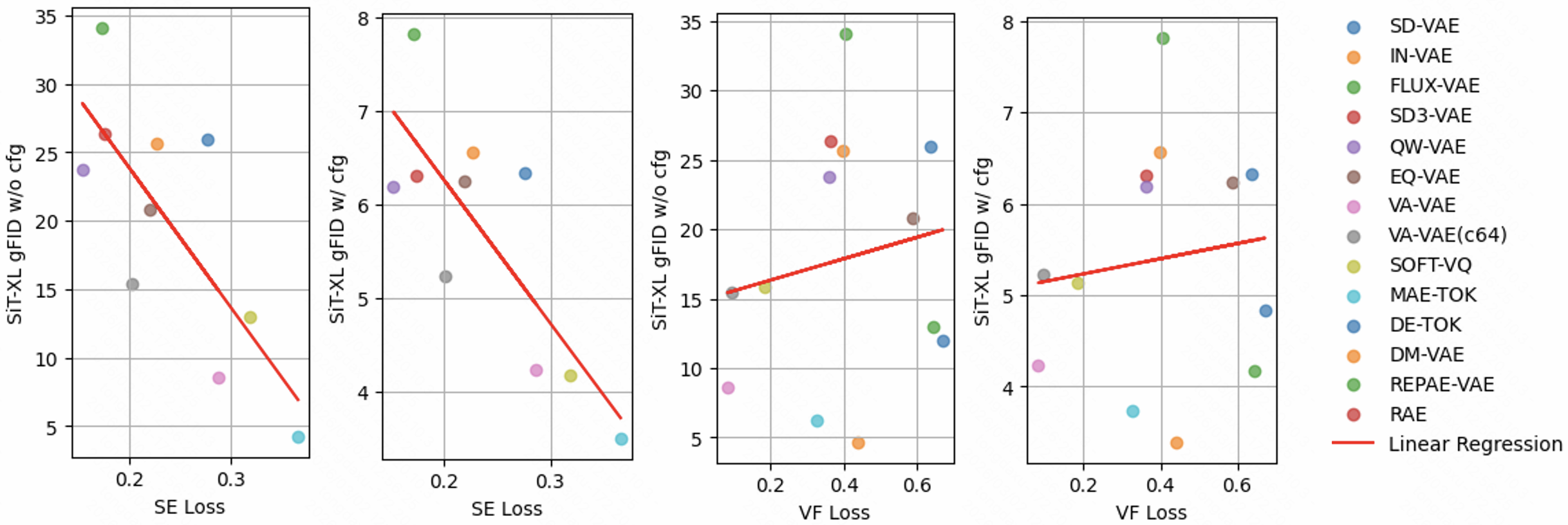}
    \includegraphics[width=\linewidth]{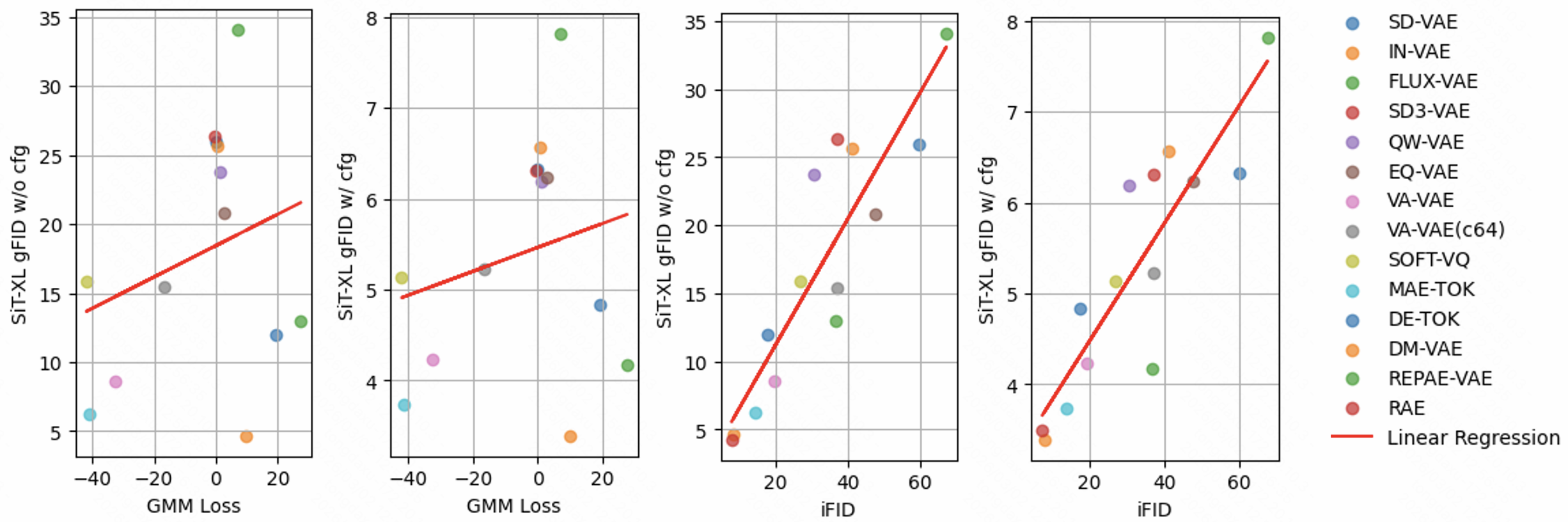}
\caption{The relationship between non-reconstruction metrics and gFID for SiT/XL. It is shown that our iFID is the only metric that is strongly correlated with gFID.}
\label{fig:rnrec}
\end{figure}

\subsection{Visualization}
In Figure~\ref{fig:viz2}, we provide additional visualizations of the source image $x$, the reconstructed image $g(z)$, the nearest neighbour image $g(\textup{NN}(z))$, and the interpolated image $g(\hat{z})$. For VAEs optimized for reconstruction, the nearest neighbours $\textup{NN}(z)$ are semantically different from $z$, and the interpolations $\hat{z}$ are not realistic. In contrast, for VAEs optimized with diffusion, the nearest neighbours $\textup{NN}(z)$ are semantically similar to $z$, and the interpolations $\hat{z}$ are more realistic.

\begin{figure}[thb]
\centering
    \includegraphics[width=0.9\linewidth]{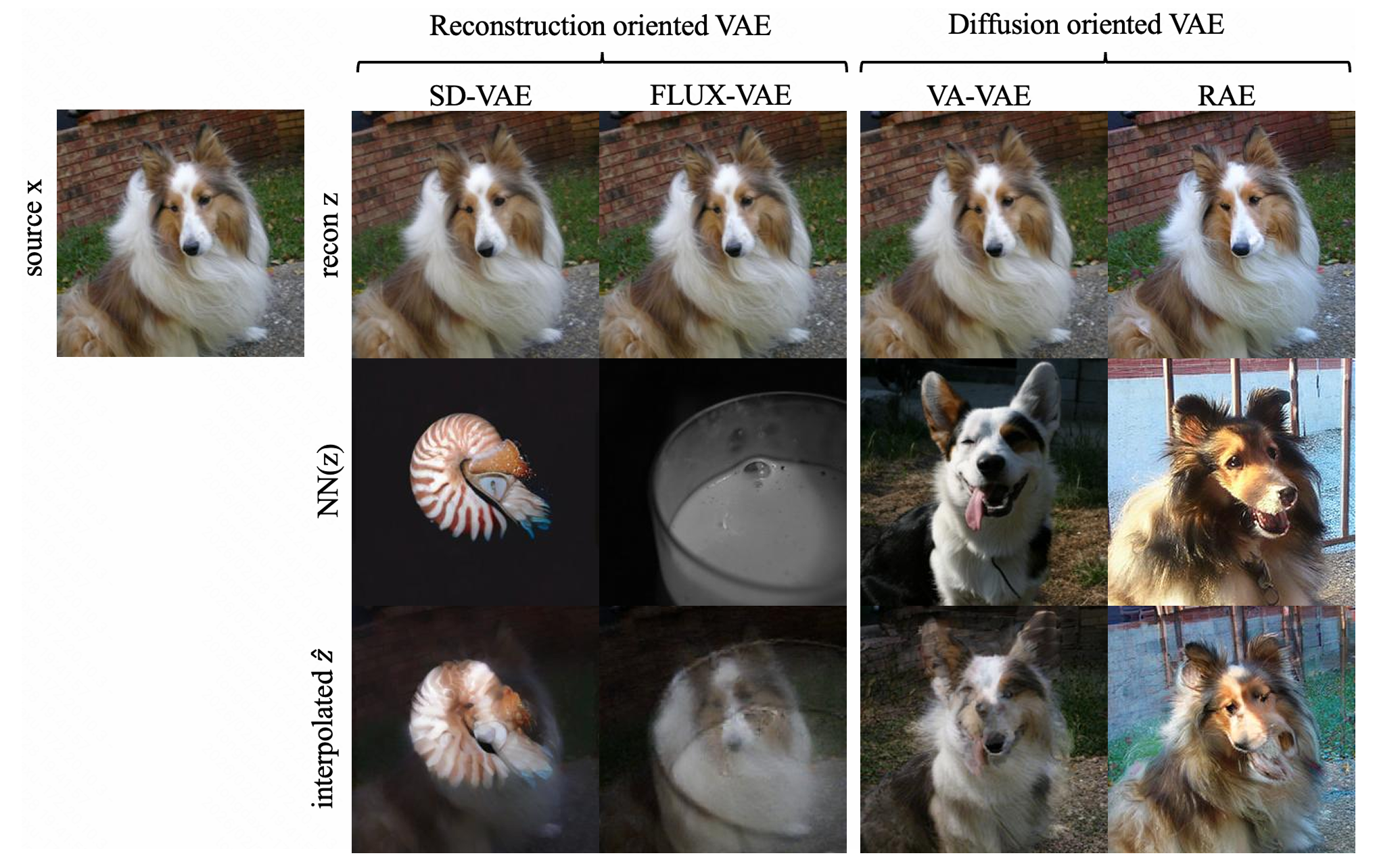}
    \includegraphics[width=0.9\linewidth]{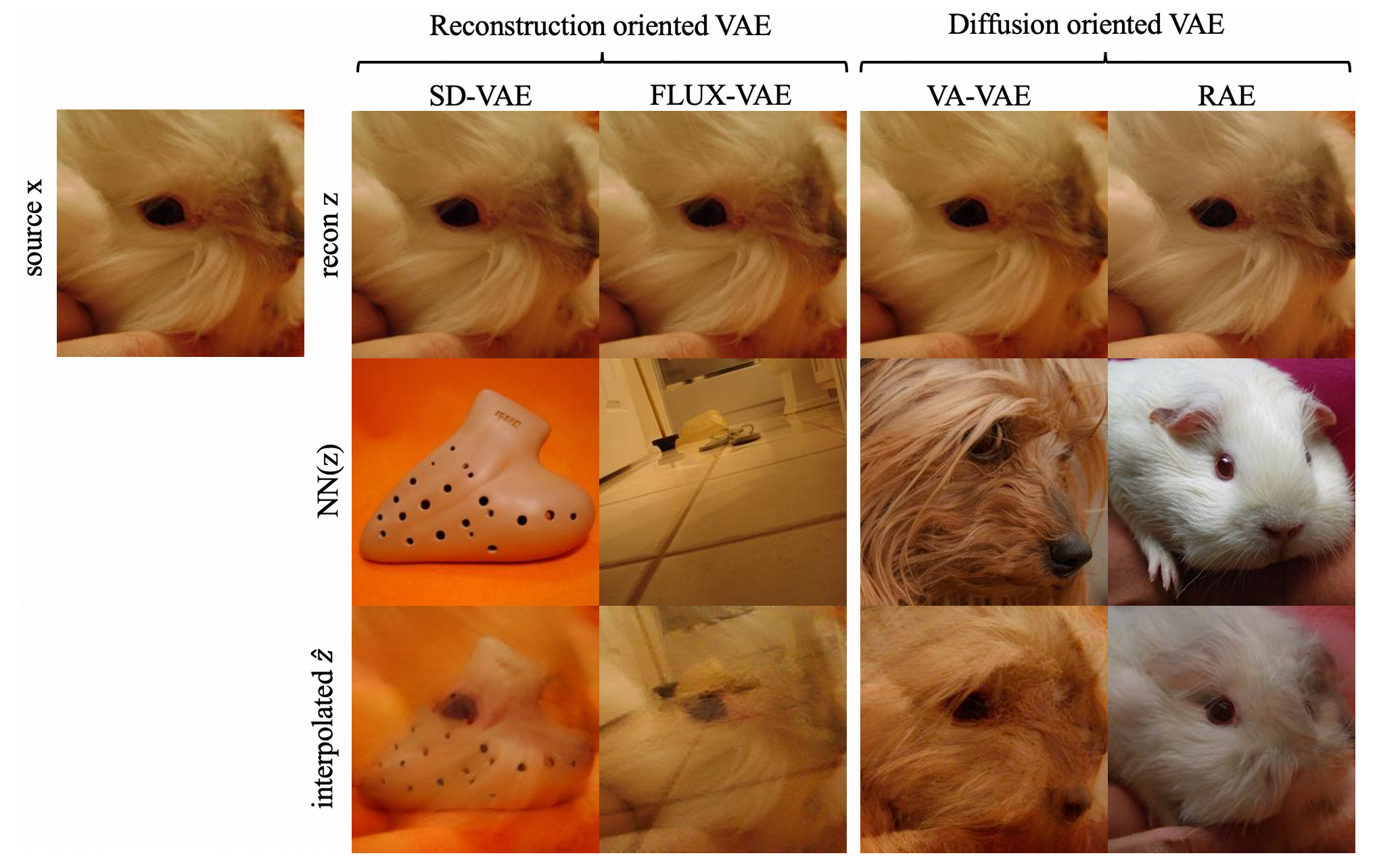}
\caption{Visualization of the decoded nearest neighbour latent NN($z$) and the interpolated latent $\hat{z}$. For reconstruction-oriented VAEs, NN($z$) is semantically different from $z$, and the interpolated $\hat{z}$ are invalid images. In contrast, for diffusion-oriented VAEs, NN($z$) is semantically similar to $z$, and the interpolated $\hat{z}$ are realistic images.}
\label{fig:viz2}
\end{figure}

\section{Additional Discussions}
\label{app:disc}
\subsection{Explanation on Diffusibility of Previous VAEs}

\textbf{Pure Reconstruction VAEs} VAEs trained solely with a reconstruction metric exhibit the worst generative performance (SD-VAE, IN-VAE, FLUX-VAE, SD3-VAE, QW-VAE) \cite{yang2025latent,flux2024,esser2024scaling,wu2025qwen,Yao2025ReconstructionVG}. According to our findings, reconstruction loss results in a more separable latent space, thereby leading to degraded generative performance.

\textbf{VAEs with Signal Processing Regularization} Several previous VAEs incorporate signal processing regularization to constrain the shape of the latent space by suppressing high-frequency information (EQ-VAE, FT-SE, SSVAE) \cite{Kouzelis2025EQVAEER,Skorokhodov2025ImprovingTD,liu2025delving}. Our findings indicate that the performance of diffusion models is primarily determined by the interpolability of the latent space. Suppressing high-frequency components in the latent space enhances latent smoothness, thereby improving interpolability.

\textbf{VAEs with Vision Foundation Models} Many previous VAEs incorporate vision foundation models as alignment mechanisms or encoders (e.g., VA-VAE, SOFT-VQ, MAE-TOK, DE-TOK, DM-VAE, REPAE-VAE, RAE, VMAE) \cite{Yao2025ReconstructionVG,chen2025softvq,Chen2025MaskedAA,ye2025distribution,Leng2025REPAEUV,Yu2024RepresentationAF,yang2025latent,lee2025latent}. According to our findings, the performance of diffusion models is primarily determined by the interpolability of the latent space. The latent space of these vision foundation models, such as DINOv2 \cite{oquab2023dinov2}, MAE \cite{he2022masked}, and Denoising Auto-encoder \cite{Vincent2011ACB}, is semantically continuous and interpolable. This explains why adopting these vision foundation models improves generation quality.

\subsection{Practical Guidance to Diffusion VAEs}

Although there is no direct way to optimize iFID, our finding tells us that reconstruction loss is harmful to generation as it encourages separable and islolated latent space. The best practical guideline we can provide is do not train the encoder with reconstruction loss. Besides, our result that RAE \cite{zheng2025diffusion} achieves best gFID w/o cfg also support this claim.

\subsection{Relationship to Discrete Tokenizers}

As early as 2022, researchers found that using vision foundation models as the encoder for vector quantized variational autoencoders (VQ-VAE) improves autoregressive image generation \cite{peng2022beit}. Subsequently, the superiority of vision foundation model encoders over reconstruction-based VQ-VAEs has been verified in numerous studies \cite{wang2024image,wu2025janus,xiong2025gigatok,geng2025x}. However, similar to VAEs in diffusion models, metrics on VQ-VAE for predicting the generation performance of autoregressive models remain under-explored. It is possible to extend iFID to VQ-VAE and evaluate its relationship to the generation performance of autoregressive models. Nevertheless, we leave this topic to future work.

\subsection{Impact Statement}
The approach proposed in this paper focus on evaluating the latent space used for diffusion model. As the proposed method is essentially not a generative model, the ethic concerns is not obvious. Nevertheless, the proposed metric might be used to improve the VAE in diffusion model, leading to more effective disinformation.

\clearpage
\newpage
\input{checklist.tex}

\end{document}

%% file: math_commands.tex
\usepackage{amsmath,amsfonts,bm}

\def\eqref#1{equation~\ref{#1}}
\def\1{\bm{1}}

\DeclareMathAlphabet{\mathsfit}{\encodingdefault}{\sfdefault}{m}{sl}
\SetMathAlphabet{\mathsfit}{bold}{\encodingdefault}{\sfdefault}{bx}{n}

%% file: checklist.tex
\section*{NeurIPS Paper Checklist}

%%% BEGIN INSTRUCTIONS %%%
%%% END INSTRUCTIONS %%%

\begin{enumerate}

\item {\bf Claims}
    \item[] Question: Do the main claims made in the abstract and introduction accurately reflect the paper's contributions and scope?
    \item[] Answer: \answerYes{} % Replace by \answerYes{}, \answerNo{}, or \answerNA{}.
    \item[] Justification: See Abstract and Introduction section.

\item {\bf Limitations}
    \item[] Question: Does the paper discuss the limitations of the work performed by the authors?
    \item[] Answer: \answerYes{} % Replace by \answerYes{}, \answerNo{}, or \answerNA{}.
    \item[] Justification: See Discussion section.

\item {\bf Theory assumptions and proofs}
    \item[] Question: For each theoretical result, does the paper provide the full set of assumptions and a complete (and correct) proof?
    \item[] Answer: \answerYes{} % Replace by \answerYes{}, \answerNo{}, or \answerNA{}.
    \item[] Justification: See Appendix~\ref{app:pf}.

    \item {\bf Experimental result reproducibility}
    \item[] Question: Does the paper fully disclose all the information needed to reproduce the main experimental results of the paper to the extent that it affects the main claims and/or conclusions of the paper (regardless of whether the code and data are provided or not)?
    \item[] Answer: \answerYes{} % Replace by \answerYes{}, \answerNo{}, or \answerNA{}.
    \item[] Justification: See Experiment Setup section.

\item {\bf Open access to data and code}
    \item[] Question: Does the paper provide open access to the data and code, with sufficient instructions to faithfully reproduce the main experimental results, as described in supplemental material?
    \item[] Answer: \answerYes{} % Replace by \answerYes{}, \answerNo{}, or \answerNA{}.
    \item[] Justification: Code is provided in supplementary material.

\item {\bf Experimental setting/details}
    \item[] Question: Does the paper specify all the training and test details (e.g., data splits, hyperparameters, how they were chosen, type of optimizer) necessary to understand the results?
    \item[] Answer: \answerYes{} % Replace by \answerYes{}, \answerNo{}, or \answerNA{}.
    \item[] Justification: See Experimental Setup section.

\item {\bf Experiment statistical significance}
    \item[] Question: Does the paper report error bars suitably and correctly defined or other appropriate information about the statistical significance of the experiments?
    \item[] Answer: \answerNo{} % Replace by \answerYes{}, \answerNo{}, or \answerNA{}.
    \item[] Justification: Error bar is too expensive and is highly uncommon for our community.

\item {\bf Experiments compute resources}
    \item[] Question: For each experiment, does the paper provide sufficient information on the computer resources (type of compute workers, memory, time of execution) needed to reproduce the experiments?
    \item[] Answer: \answerYes{} % Replace by \answerYes{}, \answerNo{}, or \answerNA{}.
    \item[] Justification: See Experimental Setup section.
    
\item {\bf Code of ethics}
    \item[] Question: Does the research conducted in the paper conform, in every respect, with the NeurIPS Code of Ethics \url{https://neurips.cc/public/EthicsGuidelines}?
    \item[] Answer: \answerYes{} % Replace by \answerYes{}, \answerNo{}, or \answerNA{}.
    \item[] Justification: The research conform code of ethnics.

\item {\bf Broader impacts}
    \item[] Question: Does the paper discuss both potential positive societal impacts and negative societal impacts of the work performed?
    \item[] Answer: \answerYes{} % Replace by \answerYes{}, \answerNo{}, or \answerNA{}.
    \item[] Justification: See Appendix~\ref{app:disc}.
    
\item {\bf Safeguards}
    \item[] Question: Does the paper describe safeguards that have been put in place for responsible release of data or models that have a high risk for misuse (e.g., pre-trained language models, image generators, or scraped datasets)?
    \item[] Answer: \answerNA{} % Replace by \answerYes{}, \answerNo{}, or \answerNA{}.
    \item[] Justification: No model nor data is released.

\item {\bf Licenses for existing assets}
    \item[] Question: Are the creators or original owners of assets (e.g., code, data, models), used in the paper, properly credited and are the license and terms of use explicitly mentioned and properly respected?
    \item[] Answer: \answerYes{} % Replace by \answerYes{}, \answerNo{}, or \answerNA{}.
    \item[] Justification: See Experimental Setup section.

\item {\bf New assets}
    \item[] Question: Are new assets introduced in the paper well documented and is the documentation provided alongside the assets?
    \item[] Answer: \answerYes{} % Replace by \answerYes{}, \answerNo{}, or \answerNA{}.
    \item[] Justification: See supplementary material.

\item {\bf Crowdsourcing and research with human subjects}
    \item[] Question: For crowdsourcing experiments and research with human subjects, does the paper include the full text of instructions given to participants and screenshots, if applicable, as well as details about compensation (if any)? 
    \item[] Answer: \answerNA{} % Replace by \answerYes{}, \answerNo{}, or \answerNA{}.
    \item[] Justification: Does not involve crowdsourcing nor human subjects.

\item {\bf Institutional review board (IRB) approvals or equivalent for research with human subjects}
    \item[] Question: Does the paper describe potential risks incurred by study participants, whether such risks were disclosed to the subjects, and whether Institutional Review Board (IRB) approvals (or an equivalent approval/review based on the requirements of your country or institution) were obtained?
    \item[] Answer: \answerNA{} % Replace by \answerYes{}, \answerNo{}, or \answerNA{}.
    \item[] Justification: Does not involve crowdsourcing nor human subjects.

\item {\bf Declaration of LLM usage}
    \item[] Question: Does the paper describe the usage of LLMs if it is an important, original, or non-standard component of the core methods in this research? Note that if the LLM is used only for writing, editing, or formatting purposes and does \emph{not} impact the core methodology, scientific rigor, or originality of the research, declaration is not required.
    %this research? 
    \item[] Answer: \answerNA{} % Replace by \answerYes{}, \answerNo{}, or \answerNA{}.
    \item[] Justification: LLMs are used only for editing.

\end{enumerate}